\documentclass[]{article}
\usepackage{fullpage, url, graphicx, subfigure}
\usepackage{amsmath, amsthm, amssymb}
\usepackage{fancyhdr}
\usepackage{mathrsfs}
\usepackage{epsfig}
\usepackage{graphics}
\usepackage{epsf}
\usepackage{amsmath, amsthm, amssymb, multirow, paralist, mathtools}
\usepackage{fullpage}
\usepackage{url}
\usepackage{algorithm,algorithmic}
\newtheorem{prop}{Proposition}
\newtheorem{definition}{Definition}
\newtheorem{lemma}{Lemma}
\newtheorem{thm}{Theorem}
\newtheorem{cor}{Corollary}

\def \g  {\mathbf{g}}

\def \p {\mathbf{p}}

\def \u {\mathbf{u}}
\def \uh {\widehat{\u}}
\def \v {\mathbf{v}}
\def \vh {\widehat{\v}}
\def \w  {\mathbf{w}}
\def \wh {\widehat{\w}}
\def \wt {\mathbf{w}^t}
\def \x {\mathbf{x}}

\def \z {\mathbf{z}}

\def \B {\mathcal{B}}
\def \Bpsi {\B_{\psi}}

\def \E {\mathbb{E}}

\def \R {\mathbb{R}}

\def \V {\mathbb{V}}

\def \prox {\mathrm{prox}}
\def \sign {\mathrm{sign}}

\def \bqs \begin{eqnarray*}
\def \eqs \end{eqnarray*}

\begin{document}
\title{Stochastic Optimization  with Importance Sampling}

\author{Peilin Zhao \\
       Department of Statistics\\
       Rutgers University\\
       Piscataway, NJ, 08854, USA\\
       \emph{peilinzhao@hotmail.com}
       \and
       Tong Zhang  \\
       Department of Statistics\\
       Rutgers University\\
       Piscataway, NJ, 08854, USA\\
       \emph{tzhang@stat.rutgers.edu}
}
\date{}

\maketitle

\begin{abstract}
Uniform sampling of training data has been commonly used in traditional stochastic optimization algorithms such as Proximal Stochastic Gradient Descent (prox-SGD) and Proximal Stochastic Dual Coordinate Ascent (prox-SDCA). Although uniform sampling can guarantee that the sampled stochastic quantity is an unbiased estimate of the corresponding true quantity, the resulting estimator may have a rather high variance, which negatively affects the convergence of the underlying optimization procedure. In this paper we study stochastic optimization with importance sampling, which improves the convergence rate by reducing the stochastic variance. Specifically, we study prox-SGD (actually, stochastic mirror descent) with importance sampling and prox-SDCA with importance sampling. For prox-SGD, instead of adopting uniform sampling throughout the training process, the proposed algorithm employs importance sampling to minimize the variance of the stochastic gradient. For prox-SDCA, the proposed importance sampling scheme aims to achieve higher expected dual value at each dual coordinate ascent step. We provide extensive theoretical analysis to show that the convergence rates with the proposed importance sampling methods can be significantly improved under suitable conditions both for prox-SGD and for prox-SDCA. Experiments are provided to verify the theoretical analysis.
\end{abstract}

%============================================================================================
%
%  Introduction
%
%============================================================================================

\section{Introduction}
Stochastic optimization has been extensively studied in the machine learning community~\cite{DBLP:conf/icml/Zhang04,rakhlin2011making,shamir2013stochastic,duchi2009efficient,luo1992convergence,mangasarian1999successive,hsieh2008dual,DBLP:journals/jmlr/Shalev-ShwartzT11,DBLP:journals/corr/abs-1207-4747,DBLP:journals/siamjo/Nesterov12,DBLP:journals/corr/abs-1211-2717,DBLP:journals/jmlr/ShaiTong13,shalev2012proximal}. In general, at every step, a traditional stochastic optimization method will sample one training example or one dual coordinate uniformly at random from the training data, and then update the model parameter using the sampled example or dual coordinate. Although uniform sampling simplifies the analysis, it is insufficient because it may introduce a very high variance of the sampled quantity, which will negatively affect the convergence rate of the resulting optimization procedure. In this paper we study stochastic optimization with importance sampling, which reduces the stochastic variance to significantly improve the convergence rate. Specifically, this paper focus on importance sampling techniques for Proximal Stochastic Gradient Descent (prox-SGD) (actually general proximal stochastic mirror descent)~\cite{duchi2009efficient,DBLP:conf/colt/DuchiSST10} and Proximal Stochastic Dual Coordinate Ascent (prox-SDCA)~\cite{shalev2012proximal}.

For prox-SGD, the traditional algorithms such as Stochastic Gradient Descent (SGD) sample training examples uniformly at random during the entire learning process, so that the stochastic gradient is an unbiased estimation of the true gradient~\cite{DBLP:conf/icml/Zhang04,rakhlin2011making,shamir2013stochastic,duchi2009efficient}. However, the variance of the resulting stochastic gradient estimator may be very high since the stochastic gradient can vary significantly over different examples.
In order to improve convergence, this paper proposes a sampling distribution and the corresponding unbiased importance weighted gradient estimator that achieves minimal variance. To this end, we analyze the relation between the variance of stochastic gradient and the sampling distribution. We show that to minimize the variance, the optimal sampling distribution should be roughly proportional to the norm of the stochastic gradient.
%However, it is impractical to compute this distribution, since it need compute all the gradients.
To simplify computation,  we also consider the use of upper bounds for the norms. Our theoretical analysis shows that under certain conditions, the proposed sampling method can significantly improve the convergence rate, and our results include the existing theoretical results for uniformly sampled prox-SGD and SGD as special cases.

Similarly for prox-SDCA, the traditional approach such as Stochastic Dual Coordinate Ascent (SDCA)~\cite{DBLP:journals/jmlr/ShaiTong13} picks
a coordinate to update by sampling the training data uniformly at random \cite{luo1992convergence,mangasarian1999successive,hsieh2008dual,DBLP:journals/jmlr/Shalev-ShwartzT11,DBLP:journals/corr/abs-1207-4747,DBLP:journals/siamjo/Nesterov12,DBLP:journals/corr/abs-1211-2717,DBLP:journals/jmlr/ShaiTong13,shalev2012proximal}.
It was shown recently that SDCA and prox-SDCA algorithm with uniform random sampling converges much faster than a fixed cyclic ordering~\cite{DBLP:journals/jmlr/ShaiTong13,shalev2012proximal}.
However, this paper shows that if we employ an appropriately defined importance sampling strategy, the convergence could be further improved. To find the optimal sampling distribution, we analyze the connection between the expected increase of dual objective and the sampling distribution,
and obtain the optimal solution that depends on the smooth parameters of the loss functions.
Our analysis shows that under certain conditions, the proposed sampling method can significantly improve the convergence rate. In addition, our theoretical results include the existing results for uniformly sampled prox-SDCA and SDCA as special cases.

The rest of this paper is organized as follows. Section~\ref{sec:related} reviews the related work. Section~\ref{sec:prelim} presents some preliminaries. In section~\ref{sec:imp-sample}, we will study stochastic optimization with importance sampling. Section~\ref{sec:application} lists several applications for the proposed algorithms. Section~\ref{sec:experiment} gives our empirical evaluations. Section~\ref{sec:conclusion} concludes the paper.

%============================================================================================
%
%  Related Work
%
%============================================================================================

\section{Related Work}
\label{sec:related}

We review some related work on Proximal Stochastic Gradient Descent (including more general proximal stochastic mirror descent) and Proximal Stochastic Dual Coordinate Ascent.

In recent years Proximal Stochastic Gradient Descent has been extensively studied~\cite{duchi2009efficient,DBLP:conf/colt/DuchiSST10}. As a special case of prox-SGD, Stochastic Gradient Descent has been extensively studied in stochastic approximation theory~\cite{kushner2003stochastic}; however these results are often asymptotic, so there is no explicit bound in terms of $T$. Later on, finite sample convergence rate of SGD for solving linear prediction problem were studied by a number of authors~\cite{DBLP:conf/icml/Zhang04,DBLP:conf/icml/Shalev-ShwartzSS07}. In general prox-SGD can achieve a convergence rate of $O(1/\sqrt{T})$ for convex loss functions, and a convergence rate of $O(\log T/T)$ for strongly convex loss functions, where $T$ is the number of iterations of the algorithm.
More recently, researchers have improved the previous bound to $O(1/T)$ by $\alpha$-Suffix Averaging ~\cite{rakhlin2011making}, which means instead of returning the average of the entire sequence of classifiers, the algorithm will average and return just an $\alpha$-suffix: the average of the last $\alpha$ fraction of the whole sequence of classifiers. In practice it may be difficult for users to decide when to compute the $\alpha$-suffix. To solve this issue, a polynomial decay averaging strategy is proposed by~\cite{shamir2013stochastic}, which will decay the weights of old individual classifiers polynomially and also guarantee a $O(1/T)$ convergence bound.

For Proximal Stochastic Dual Coordinate Ascent~\cite{shalev2012proximal}, Shalev-Shwartz and Zhang recently proved that the algorithm  achieves a convergence rate of $O(1/T)$ for Lipschitz loss functions, and enjoys a linear convergence rate of $O(\exp(-O(T)))$ for smooth loss functions. For structural SVM, a similar result was also obtained in \cite{DBLP:journals/corr/abs-1207-4747}. Several others researchers~\cite{mangasarian1999successive,hsieh2008dual}  have studied the convergence behavior of the related non-randomized DCA (dual coordinate ascent) algorithm for SVM, but could only obtain weaker convergence results. The related randomized coordinate descent method has been investigated by some other authors~\cite{DBLP:journals/jmlr/Shalev-ShwartzT11,DBLP:journals/siamjo/Nesterov12,richtarik2012iteration}. However, when applied to SDCA, the analysis can only lead to a convergence rate of the dual objective value while we are mainly interested in the convergence of primal objective in machine learning applications. Recently, Shai Shalev-Shwartz and Tong Zhang has resolved this issue by providing a primal-dual analysis that showed a linear convergence rate $O(\exp(- O(T)))$ of the duality gap for SDCA with smooth loss function~\cite{DBLP:journals/jmlr/ShaiTong13}.

Although both prox-SGD and prox-SDCA have been extensively studied, most of the existing work only considered the uniform sampling scheme during the entire learning process. Recently, we noticed that~\cite{needell2014stochastic} Deanna Needell et. al. considered importance sampling for stochastic gradient descent, where they suggested similar or the same sampling distributions. Strohmer and Vershynin~\cite{strohmer2009randomized} proposed a variant of the Kaczmarz method (an iterative method for solving systems of linear equations) which selects rows with probability proportional to their squared norm. It is pointed out that, this algorithm is actually a SGD algorithm with importance sampling~\cite{needell2014stochastic}. However, we have studied importance sampling for more general composite objectives and more general proximal stochastic gradient descent, i.e., proximal stochastic mirror descent which covers their algorithms as special cases.  Furthermore, we have also studied prox-SDCA with importance sampling, which is not covered by their study. In addition, Xiao and Zhang~\cite{xiao2014proximal} have also proposed a proximal stochastic gradient method with progressive variance reduction, where they also provide importance sampling strategy for only smooth loss functions, which is the same with ours. Because our analysis is based on the basic version of stochastic gradient (mirror) descent, the  convergence rate is worse than the linear rates in SAG~\cite{roux2012stochastic} and SVRG~\cite{xiao2014proximal} for smooth strongly convex objective functions. However, our main concern is on the effectiveness of importance sampling, which could be applied to many other gradient based algorithms.

We shall mention that for coordinate descent, some researchers have recently considered non-uniform sampling strategies~\cite{nesterov2012efficiency,lee2013efficient}, but their results cannot be directly applied to proximal SDCA which we are interested in here. The reasons are several-folds. The primal-dual analysis of prox-SDCA in this paper is analogous to that of ~\cite{DBLP:journals/jmlr/ShaiTong13}, which directly implies a convergence rate for the duality gap. The proof techniques rely on the structure of the regularized loss minimization, which can not be applied to general primal coordinate descent. The suggested distribution of the primal coordinate descent is propositional to the smoothness constant of every coordinate, while the distribution of prox-SDCA is propositional to a constant plus the smoothness constant of the primal individual loss function, which is the inverse of the strongly convex constant of the dual coordinate. These two distributions are quite different. In addition, we also provide an importance sampling distribution when the individual loss functions are Lipschitz. We also noticed that a mini-batch SDCA~\cite{DBLP:conf/nips/ShaiTong} and an accelerated version of prox-SDCA~\cite{DBLP:conf/icml/Shalev-Shwartz014} were studied recently by Shalev-Shwartz and Zhang. The accelerated version in~\cite{DBLP:conf/icml/Shalev-Shwartz014}  uses an inner-outer-iteration strategy, where the inner iteration is the standard prox-SDCA procedure.
Therefore the importance sampling results of this paper can be directly applied to the
accelerated prox-SDCA because
the convergence of inner iteration is faster than that of uniform sampling.
Therefore in this paper we will only focus on showing the effectiveness of importance sampling for the unaccelerated prox-SDCA.

Related to this paper, non-uniform sampling in the online setting is related to selective sampling, which can be regarded as a form of online active learning which has been extensively studied in the literature~\cite{DBLP:conf/colt/Cesa-BianchiCG03,DBLP:conf/nips/CavallantiCG08,DBLP:conf/icml/Cesa-BianchiGO09,DBLP:conf/icml/OrabonaC11,DBLP:journals/ml/CavallantiCG11}. Similar to importance sampling in stochastic optimization, selective sampling also works in iterations. However the purposes are quite different.
Specifically, selective sampling draws unlabeled instances uniformly at random from a
fixed distribution and decides which samples to label --- the goal is to reduce the number
of labels needed to achieve a certain accuracy; importance sampling considered in this paper
does not reduce the number of labels needed, and the goal is to reduce the training time.

%============================================================================================
%
%  Preliminaries
%
%============================================================================================

\section{Preliminaries}
\label{sec:prelim}
Here, we briefly introduce some key definitions and propositions that are useful throughout the paper (for details, please refer to~\cite{borwein2006convex} ). We consider vector functions: $\phi: \R^d\rightarrow\R$.
\begin{definition}
For $\sigma\ge 0$, a function $\phi: \R^d\rightarrow\R$  is $\sigma$-strongly convex with respect to (w.r.t.) a norm $\|\cdot\|$, if for all $\u,\v \in \R^d$, we have
\begin{eqnarray*}
\phi(\u) \ge \phi(\v) + \nabla\phi(\v)^\top(\u-\v) + \frac{\sigma}{2}\|\u-\v\|^2 ,
\end{eqnarray*}
or  equivalently, $\forall s\in[0,1]$
\begin{eqnarray*}
\phi(s\u+(1-s)\v)\le s\phi(\u) + (1-s)\phi(\v) - \frac{\sigma s(1-s)}{2}\|\u-\v\|^2 .
\end{eqnarray*}
\end{definition}
For example, $\phi(\w)=\frac{1}{2}\|\w\|_2^2$ is $1$-strongly convex w.r.t. $\|\cdot\|_2$.
\begin{definition}
A function $\phi: \R^d\rightarrow\R$ is $L$-Lipschitz w.r.t. a norm $\|\cdot\|$, if for all $\u,\v \in\R^d$, we have
\begin{eqnarray*}
|\phi(\u)-\phi(\v)|\le L \|\u-\v\|.
\end{eqnarray*}
\end{definition}
\begin{definition}
A function $\phi: \R^d\rightarrow\R$ is $(1/\gamma)$-smooth if it is differentiable and its gradient is $(1/\gamma)$-Lipschitz, or, equivalently for all $\u,\v\in\R^d$, we have
\begin{eqnarray*}
\phi(\u) \le \phi(\v) + \nabla\phi(\v)^\top(\u-\v) + \frac{1}{2\gamma}\|\u-\v\|^2 .
\end{eqnarray*}
\end{definition}
For example, $\phi(\w)=\frac{1}{2}\|\w\|_2^2$ is $1$-smooth w.r.t. $\|\cdot\|_2$.
\begin{prop}
If $\phi$ is $(1/\gamma)$-smooth with respect to a norm $\|\cdot\|_P$, then its dual function $\phi^*$ is $\gamma$-strongly convex with respect to its dual norm $\|\cdot\|_D$, where
\begin{eqnarray*}
\phi^*(\v)=\sup_\w(\v^\top\w-\phi(\w)) ,
\end{eqnarray*}
and the dual norm is defined as
\begin{eqnarray*}
\|\v\|_D=\sup_{\|\w\|_P=1}\v^\top\w .
\end{eqnarray*}
\end{prop}
For example, the dual norm of $\|\cdot\|_2$ is itself. The dual norm of $\|\cdot\|_1$ is $\|\cdot\|_\infty$. The dual norm of $\|\cdot\|_p$ is $\|\cdot\|_q$, where $1/q+1/p=1$.
\begin{definition}
Let $\psi:\R^d\rightarrow\R$ be a continuously-differentiable real-valued and strictly convex function.  Then the Bregman divergence associated with $\psi$ is
\begin{eqnarray*}
\B_{\psi}(\u,\v)=\psi(\u)-\psi(\v) - \langle\nabla\psi(\v),\u-\v\rangle,
\end{eqnarray*}
which is the difference between the value of $\psi$ at $\u$ and the value of the first-order Taylor expansion of $\psi$ around $\v$ evaluated at $\u$.
\end{definition}
Throughout, $\psi$ denotes a continuously differentiable function that is $\sigma$-strongly convex w.r.t. a norm $\|\cdot\|$, so that $\Bpsi(\u,\v)\ge \frac{\sigma}{2}\|\u-\v\|^2$.

\begin{definition}
A function $f:\R^d \rightarrow \R$ is $\mu$-strongly convex with respect to a differentiable function $\psi$, if  for any $\u,\v$, we have
\begin{eqnarray*}
f(\u) \ge f(\v) + \langle \nabla f(\v), \u-\v \rangle + \mu \Bpsi(\u,\v).
 \end{eqnarray*}
\end{definition}
For example, when $\psi(\w)=\frac{1}{2}\|\w\|_2^2$, we recover the usual definition of strongly convexity.

\begin{definition}
A function $f:\R^d \rightarrow \R$ is $(1/\gamma)$-smooth with respect to a differentiable function $\psi$, if  for any $\u,\v$, we have
\begin{eqnarray*}
f(\u) \le f(\v) + \langle \nabla f(\v), \u-\v \rangle + (1/\gamma) \Bpsi(\u,\v).
 \end{eqnarray*}
\end{definition}
%Throughout this paper, we will denote $\|\cdot\|$ as $\|\cdot\|_2$ for simplicity.

%============================================================================================
%
%  Proximal Stochastic Gradient Descent with Importance Sampling
%
%============================================================================================

\section{Stochastic Optimization with Importance Sampling}
\label{sec:imp-sample}
We consider the following generic optimization problem associated with regularized loss minimization of linear predictors. Let $\phi_1,\phi_2,\ldots,\phi_n$ be $n$ vector functions from $\R^d$ to $\R$. Our goal is to find an approximate solution of the following optimization problem
\begin{eqnarray}\label{eqn:primal-objective}
\min_{\w\in\R^d} P(\w):=\underbrace{\frac{1}{n}\sum^n_{i=1}\phi_i(\w)}_{f(\w)}+ \lambda r(\w),
\end{eqnarray}
where $\lambda >0$ is a regularization parameter, and $r$ is a regularizer.

For example, given examples $(\x_i,y_i)$ where $\x_i\in\R^d$ and $y_i\in\{-1,+1\}$, the Support Vector Machine problem is obtained by setting $\phi_i(\w)=[1-y_i\x_i^\top\w]_+$, $[z]_+=\max(0,z)$, and $r(\w)=\frac{1}{2}\|\w\|_2^2$. Regression problems also fall into the above. For example, ridge regression is obtained by setting $\phi_i(\w)=(y_i-\x_i^\top\w)^2$ and $r(\w)=\frac{1}{2}\|\w\|_2^2$, lasso is obtained by setting $\phi_i(\w)=(y_i-\x_i^\top\w)^2$ and $r(\w)=\|\w\|_1$.

Let $\w^*$ be the optimum of~\eqref{eqn:primal-objective}. We say that a solution $\w$ is $\epsilon_P$-sub-optimal if $P(\w)-P(\w^*)\le \epsilon_P$. We analyze the convergence rates of the proposed algorithms with respect to the number of iterations.

\subsection{Proximal Stochastic Gradient Descent with Importance Sampling}
In this subsection, we would  consider the proximal stochastic mirror descent with importance sampling. Because proximal stochastic mirror descent is general version of proximal stochastic gradient descent (prox-SGD), we will abuse SGD to replace stochastic mirror descent.

If we directly apply full or stochastic gradient descent to  the  optimization problem~\eqref{eqn:primal-objective}, the solution may not satisfy some desirable property. For example, when $r(\w)=\|\w\|_1$, the optimal solution of the problem~\eqref{eqn:primal-objective} should be sparse, and we would like the approximate solution to be sparse as well. However, if we directly use stochastic (sub-)gradient descent, then the resulting solution will not achieve sparsity~\cite{duchi2009efficient}.

To effectively and efficiently solve the optimization problem~\eqref{eqn:primal-objective},  a well known method  is the proximal stochastic (sub)-gradient descent. Specifically, Proximal Stochastic Gradient Descent works in iterations. At each iteration $t=1,2,\ldots$, $i_t$ will be uniformly randomly draw from $\{1,2,\ldots,n\}$, and  the iterative solution will be updated  according to the formula
\begin{eqnarray}\label{eqn:psgd}
\w^{t+1} = \arg\min_{\w}\left[\langle \nabla \phi_{i_t}(\wt), \w\rangle + \lambda r(\w) + \frac{1}{\eta_t}\Bpsi(\w,\wt) \right] .
\end{eqnarray}
where $\Bpsi$ is a Bregman divergence and $\nabla \phi_{i_t}(\wt)$ denotes an arbitrary (sub-)gradient of $\phi_{i_t}$. Intuitively, this method works by minimizing a first-order approximation of the function $\phi_{i_t}$ at the current iterate $\wt$ plus the regularizer $\lambda r(\w)$, and forcing the next iterate $\w^{t+1}$ to lie close to $\wt$. The step size $\eta_t$ is used to controls the trade-off between these two objectives.Because the expectation of $\nabla \phi_{i_t}(\wt)$ is the same with $\nabla f(\wt)$, i.e., $\E[\nabla \phi_{i_t}(\wt)|\wt]=\frac{1}{n}\sum^n_{i=1}\nabla \phi_i(\wt)=\nabla f(\wt)$, the optimization problem~(\ref{eqn:psgd}) is an unbiased estimation of that for the proximal gradient descent.

We assume that the exact solution of the above optimization~(\ref{eqn:psgd}) can be efficiently solved. For example, when $\psi(\w)=\frac{1}{2}\|\w\|^2_2$, we have $\Bpsi(\u,\v)=\frac{1}{2}\|\u-\v\|_2^2$, and the above optimization will produce the $t+1$-th iterate as:
\begin{eqnarray*}
\w^{t+1}=\prox_{\eta_t\lambda r}\left(\wt-\eta_t\nabla \phi_{i_t}(\wt)\right),
\end{eqnarray*}
where $\prox_h(\x)=\arg\min_{\w}\Big(h(\w)+\frac{1}{2}\|\w-\x\|_2^2\Big)$. Furthermore, it is also assumed that the proximal mapping of $\eta_t\lambda r(\w)$, i.e., $\prox_{\eta_t\lambda r}(\x)$, is easy to compute. For example, when $r(\w)=\|\w\|_1$, the proximal mapping of $\lambda r(\w)$ is the following shrinkage operation
\begin{eqnarray*}
\prox_{\lambda r}(\x)=\sign(\x)\odot[|\x|-\lambda]_+,
\end{eqnarray*}
where $\odot$ is the element wise product, which can be computed in time complexity $O(d)$.

The advantage of proximal stochastic gradient descent is that each step only relies on a single derivative $\nabla\phi_{i_t}(\cdot)$, and thus the computational cost is $1/n$ of that of the standard proximal gradient descent. However, a disadvantage of the method is that the randomness introduces variance - this is caused by the fact that $\nabla \phi_{i_t}(\wt)$ equals the gradient $\nabla f(\wt)$ in expectation, but $\nabla \phi_i(\wt)$ varies with $i$. In particular, if the stochastic gradient has a large variance, then the convergence will become slow.

Now, we would like to study prox-SGD with importance sampling to reduce the variance of stochastic gradient. The idea of importance sampling is, at the $t$-th step, to assign each $i\in\{1,\ldots,n\}$ a probability $p_i^t\ge 0$ such that $\sum^n_{i=1} p_i^t=1$. We then sample $i_t$ from $\{1,\ldots, n\}$ based on probability $\p^t=(p_1^t,\ldots,p_n^t)^\top$. If we adopt this distribution,  then proximal SGD with importance sampling will work as follows:
\begin{eqnarray}\label{eqn:iprox-sgd}
\w^{t+1}=\arg\min_{\w}\Big[\langle(n p_{i_t}^t)^{-1} \nabla\phi_{i_t}(\wt), \w \rangle +\lambda r(\w)+\frac{1}{\eta_t}\Bpsi(\w,\wt)\Big],
\end{eqnarray}
which is another unbiased estimation of the optimization problem  for proximal gradient descent, because $\E[(np_{i_t}^t)^{-1}\nabla \phi_{i_t}(\wt)|\wt]=\sum^n_{i=1}p_i^t(np_i^t)^{-1}\nabla\phi_i(\wt)=\nabla f(\wt)$.

Similarly, if $\psi(\w)=\frac{1}{2}\|\w\|_2^2$, the proximal SGD with importance sampling  will produce the $t+1$-th iterate as:
\begin{eqnarray*}
\w^{t+1}=\prox_{\eta_t\lambda r}\left(\wt-\eta_t(np_{i_t}^t)^{-1} \nabla\phi_{i_t}(\wt)\right).
\end{eqnarray*}

In addition, setting the derivative of optimization function in equation~(\ref{eqn:iprox-sgd}) as zero, we can obtain the following implicit update rule for the iterative solution:
\begin{eqnarray*}
\nabla\psi(\w^{t+1})=\nabla\psi(\wt)-\eta_t(np_{i_t}^t)^{-1}\nabla\phi_{i_t}(\wt)-\eta_t\lambda\partial r(\w^{t+1}),
\end{eqnarray*}
where $\partial r(\w^{t+1})$ is a subgradient.

Now the key question that attracted us is which $\p^t$ can optimally reduce the variance of the stochastic gradient. To answer this question, we will firstly prove a lemma, that can indicates the relationship between $\p^t$ and the convergence rate of prox-SGD with importance sampling.

\begin{lemma}\label{lemma:primal-gap-t}
Let $\w^{t+1}$ be defined by the update~(\ref{eqn:iprox-sgd}). Assume that $\psi(\cdot)$ is $\sigma$-strongly convex with respect to a norm $\|\cdot\|$, and that $f$ is $\mu$-strongly convex and $(1/\gamma)$-smooth with respect to $\psi$, if $r(\w)$ is convex and $\eta_t\in(0,\gamma]$ then $\w^{t+1}$ satisfies the following inequality for any $t\ge 1$,
\begin{eqnarray*}
\E[P(\w^{t+1})-P(\w^*)] \le\frac{1}{\eta_t}\E[\Bpsi(\w^*,\wt) -\Bpsi(\w^*,\w^{t+1})] -\mu\E\Bpsi(\w^*,\wt) +  \frac{\eta_t}{\sigma}\E\V\left( (np_{i_t}^t)^{-1} \nabla\phi_{i_t}(\wt)\right),
\end{eqnarray*}
where the variance is defined as $\V( (np_{i_t}^t)^{-1} \nabla\phi_{i_t}(\wt))=\E \| (np_{i_t}^t)^{-1} \nabla \phi_{i_t}(\wt)- \nabla f(\wt)\|_*^2$, and the expectation is take with the distribution $\p^t$.
\end{lemma}
\begin{proof}
To simplify the notation, we denote $\g_t=(n p_{i_t}^t)^{-1} \nabla\phi_{i_t}(\wt)$. Because $f(\w)$ is $\mu$-strongly convex w.r.t. $\psi$, and $r(\w)$ is convex, we can derive
\begin{eqnarray*}
P(\w^*)\ge f(\wt)+\langle\nabla f(\wt), \w^* -\wt \rangle +  \mu\Bpsi(\w^*,\wt)+\lambda r(\w^{t+1})+\lambda\langle\partial r(\w^{t+1}), \w^{*}-\w^{t+1}  \rangle.
\end{eqnarray*}
Using the fact $f$ is $(1/\gamma)$-smooth w.r.t. $\psi$, we can further lower bound $f(\wt)$ by
\begin{eqnarray*}
f(\wt)\ge f(\w^{t+1}) - \langle \nabla f(\wt), \w^{t+1}-\wt \rangle - (1/\gamma)\Bpsi(\w^{t+1},\wt).
\end{eqnarray*}
Combining the above two inequalities, we have
\begin{eqnarray*}
P(\w^*)\ge P(\w^{t+1})+ \langle\nabla f(\wt)+\lambda \partial r(\w^{t+1}), \w^*-\w^{t+1} \rangle + \mu\Bpsi(\w^*,\wt) - (1/\gamma)\Bpsi(\w^{t+1},\wt).
\end{eqnarray*}
Considering the second term on the right-hand side, we have
\begin{eqnarray*}
\langle\nabla f(\wt)+\lambda \partial r(\w^{t+1}), \w^*-\w^{t+1} \rangle&=&\langle\nabla f(\wt)+[\nabla\psi(\wt)-\nabla\psi(\w^{t+1})]/\eta_t -\g_t, \w^*-\w^{t+1} \rangle\\
&=&\frac{1}{\eta_t}\langle\nabla\psi(\wt)-\nabla\psi(\w^{t+1})  , \w^*-\w^{t+1} \rangle + \langle \g_t- \nabla f(\wt), \w^{t+1}-\w^* \rangle.
\end{eqnarray*}
Combining the above two inequalities, we get
\begin{eqnarray*}
&&P(\w^*)-P(\w^{t+1})-\mu\Bpsi(\w^*,\wt)-  \langle \g_t -\nabla f(\wt)  , \w^{t+1}-\w^* \rangle\\
&&\ge  \langle\nabla f(\wt)+\lambda \partial r(\w^{t+1}), \w^*-\w^{t+1} \rangle - (1/\gamma)\Bpsi(\w^{t+1},\wt) -  \langle \g_t -\nabla f(\wt)  , \w^{t+1}-\w^* \rangle\\
&&=\frac{1}{\eta_t}\langle \nabla\psi(\wt)-\nabla\psi(\w^{t+1}), \w^*-\w^{t+1} \rangle - (1/\gamma)\Bpsi(\w^{t+1},\wt).
\end{eqnarray*}

Plugging the following equality (Lemma 11.1 from~\cite{DBLP:books/daglib/0016248})
\begin{eqnarray*}
\Bpsi(\w^*,\w^{t+1}) + \Bpsi(\w^{t+1},\wt)-\Bpsi(\w^*,\wt) = \langle \nabla\psi(\wt)-\nabla\psi(\w^{t+1}), \w^*-\w^{t+1}\rangle,
\end{eqnarray*}
into the previous inequality gives
\begin{eqnarray*}
&&P(\w^*)-P(\w^{t+1})-\mu\Bpsi(\w^*,\wt)-  \langle \g_t -\nabla f(\wt)  , \w^{t+1}-\w^* \rangle\\
&&\ge \frac{1}{\eta_t}\left[\Bpsi(\w^*,\w^{t+1}) + \Bpsi(\w^{t+1},\wt)-\Bpsi(\w^*,\wt)\right]  - (1/\gamma)\Bpsi(\w^{t+1},\wt)\\
&&\ge \frac{1}{\eta_t}\left[\Bpsi(\w^*,\w^{t+1}) -\Bpsi(\w^*,\wt)\right],
\end{eqnarray*}
where $\eta_t\in(0,\gamma]$ is used for the final inequality. Re-arranging the above inequality and taking expectation on both sides result in
\begin{eqnarray*}
\E[P(\w^{t+1})-P(\w^*)]\le\frac{1}{\eta_t}\E[\Bpsi(\w^*,\wt) -\Bpsi(\w^*,\w^{t+1})] -\mu\E\Bpsi(\w^*,\wt)-  \E\langle\g_t -\nabla f(\wt)  , \w^{t+1}-\w^* \rangle.
\end{eqnarray*}
To upper bound the last inner product term on the right-hand side, we can define the proximal full gradient update as $\wh^{t+1} = \arg\min_{\w}\left[\langle \nabla f(\wt), \w\rangle + \lambda r(\w) + \frac{1}{\eta_t}\Bpsi(\w, \wt) \right]$, which is independent with $\g_t$. Then we can bound $-  \E \langle\g_t -\nabla f(\wt)  , \w^{t+1}-\w^* \rangle$ as follows
\begin{eqnarray*}
-\E \langle\g_t -\nabla f(\wt), \w^{t+1}-\w^* \rangle&=& -\E \langle \g_t -\nabla f(\wt), \w^{t+1}-\wh^{t+1}+\wh^{t+1}-\w^* \rangle\\
&=& -\E \langle \g_t -\nabla f(\wt), \w^{t+1}-\wh^{t+1}\rangle - \E \langle \g_t -\nabla f(\wt), \wh^{t+1}-\w^* \rangle\\
&\le& \E \| \g_t -\nabla f(\wt)\|_* \|\w^{t+1}-\wh^{t+1}\|- \E \langle \g_t -\nabla f(\wt), \wh^{t+1}-\w^* \rangle\\
&\le &\E  \frac{\eta_t}{\sigma}\|\g_t-  \nabla f(\wt)\|_*^2- \E \langle \g_t -\nabla f(\wt), \wh^{t+1}-\w^* \rangle\\
&= &\E  \frac{\eta_t}{\sigma}\|(np_{i_t}^t)^{-1} \nabla \phi_{i_t}(\wt) - \nabla f(\wt)\|_*^2= \frac{\eta_t}{\sigma}\V\left((np_{i_t}^t)^{-1} \nabla \phi_{i_t}(\wt)\right),
\end{eqnarray*}
where, the first inequality is due to Cauchy-Schwartz inequality, the second inequality is due to Lemma~\ref{lem:nonexpansive}, and the last equality is because of $\E[ \langle \g_t -\nabla f(\wt), \wh^{t+1}-\w^* \rangle|\wt]=0$. Finally, plugging the above inequality into the previous one concludes the proof of this lemma.
\end{proof}

From the above analysis, we can observe that the smaller the variance, the more reduction on objective function we have. In the next subsection, we will study how to adopt importance sampling to reduce the variance. This observation will be made more rigorous below.

\subsubsection{Algorithm}

According to the result in the Lemma~\ref{lemma:primal-gap-t} , to maximize the reduction on the objective value, we should choose $\p^t$ as the solution of the following optimization
\begin{eqnarray}\label{eqn:optimization-psgd}
\min_{\p^t,p_i^t\in[0,1],\sum^n_{i=1}p_i^t=1}\V((np_{i_t}^t)^{-1}\nabla\phi_{i_t}(\wt))\Leftrightarrow \min_{\p^t,p_i^t\in[0,1],\sum^n_{i=1}p_i^t=1}\frac{1}{n^2}\sum^n_{i=1}(p_i^t)^{-1}\|\nabla\phi_i(\wt)\|_*^2.
\end{eqnarray}
It is easy to verify, that the solution of the above optimization is
\begin{eqnarray}\label{eqn:distribution-psgd}
p_i^t=\frac{\|\nabla \phi_{i}(\wt)\|_*}{\sum^n_{j=1}\|\nabla\phi_{j}(\wt)\|_*},\quad \forall i\in\{1,2,\ldots,n\}.
\end{eqnarray}

Although, this distribution can minimize the variance of the $t$-th stochastic gradient, it requires calculation of $n$ derivatives at each step, which is clearly inefficient. To solve this issue, a potential solution is to calculate the $n$ derivatives at some steps and then keep it for use for a long time. In addition the true derivatives will changes every step, it may be better to add a smooth parameter to the sampling distribution. However this solution still can be inefficient. Another more practical solution is to relax the previous optimization~\eqref{eqn:optimization-psgd} as follows
\begin{eqnarray}\label{eqn:relax-psgd}
\min_{\p^t,p_i^t\in[0,1],\sum^n_{i=1}p_i^t=1}\frac{1}{n^2}\sum^n_{i=1}(p_i^t)^{-1}\|\nabla\phi_i(\wt)\|_*^2\le \min_{\p^t,p_i^t\in[0,1],\sum^n_{i=1}p_i^t=1}\frac{1}{n^2}\sum^n_{i=1}(p_i^t)^{-1}G_i^2
\end{eqnarray}
by introducing
\begin{eqnarray*}
G_i\ge\|\nabla \phi_i(\w^t)\|_*,\quad \forall t.
\end{eqnarray*}
Then, we can approximate the distribution in equation~(\ref{eqn:distribution-psgd}) by solving the the right hand side of the inequality~\eqref{eqn:relax-psgd} as
\begin{eqnarray*}
p_i^t=\frac{G_i}{\sum^n_{j=1}G_j},\quad \forall i\in\{1,2,\ldots,n\},
\end{eqnarray*}
which is independent with $t$.

Based on the above solution, we will suggest distributions for two kinds of loss functions - Lipschitz functions and smooth functions. Firstly, if $\phi_i(\w)$ is $L_i$-Lipschitz w.r.t. $\|\cdot\|_*$, then $\|\nabla \phi_i(\w)\|_*\le L_i$ for any $\w\in\R^d$, and the suggested distribution is
\begin{eqnarray*}
p_i^t=\frac{L_i}{\sum^n_{j=1}L_j},\quad \forall i\in\{1,2,\ldots,n\}.
\end{eqnarray*}
Secondly, if $\phi_i(\w)$ is $(1/\gamma_i)$-smooth and $\|\wt\|\le R$ for any $t$, then $\|\nabla \phi_i(\w^t)\|_*\le  R/\gamma_i$, then the advised distribution is
\begin{eqnarray*}
p_i^t=\frac{\frac{1}{\gamma_i}}{\sum^n_{j=1}\frac{1}{\gamma_j}},\quad \forall i\in\{1,2,\ldots,n\}.
\end{eqnarray*}

Finally, we can summarize the proposed Proximal SGD with importance sampling in Algorithm~\ref{alg:ipsgd}.
\begin{algorithm}[htpb]
\caption{Proximal Stochastic Gradient Descent with Importance Sampling (Iprox-SGD)} \label{alg:ipsgd}
\begin{algorithmic}
\STATE {\bf Input}: $\lambda\ge 0$, the learning rates $\eta_1,\ldots,\eta_T>0$.
\STATE {\bf Initialize}:  $\w^1=0$, $\p^1=(1/n,\ldots,1/n)^\top$.
\FOR{$t=1,\ldots,T$}
\STATE Update $\p^t$;
\STATE Sample $i_t$ from $\{1,\ldots,n\}$ based on $\p^t$;
\STATE Update $\w^{t+1}=\arg\min_{\w}\left[\left\langle(n p_{i_t}^t)^{-1} \nabla\phi_{i_t}(\wt), \w \right\rangle +\lambda r(\w)+\frac{1}{\eta_t}\Bpsi(\w,\wt)\right]$;
\ENDFOR
%\STATE {\bf Output:} $\w^{T+1}$ or $\sum^{T+1}_{t=1}\alpha_t\w^t$ with $\alpha_t\in[0,1]$ and $\sum_t\alpha_t=1$
\end{algorithmic}
\end{algorithm}

\subsubsection{Analysis}
This section provides a convergence analysis of the proposed algorithm. Before presenting the results, we  make some general assumptions:
\[
r(\textbf{0})=0,\quad \text{and}\  r(\w)\ge0,\ \text{for all} \; \w.
\]
It is easy to see that these two assumptions are generally satisfied by all the well-known regularizers.

Under the above assumptions,  we first prove a convergence result for Proximal SGD with importance sampling using the previous Lemma~\ref{lemma:primal-gap-t}.
\begin{thm}\label{thm:psgd}
Let $\wt$ be generated by the proposed algorithm. Assume that $\psi(\cdot)$ is $\sigma$-strongly convex with respect to a norm $\|\cdot\|$, and that $f$ is $\mu$-strongly convex and $(1/\gamma)$-smooth with respect to $\psi$, if $r(\w)$ is convex and $\eta_t =\frac{1}{\alpha+\mu t}$ with $\alpha \ge 1/\gamma - \mu$, the following inequality holds for any $T\ge 1$,
\begin{eqnarray}\label{eqn:psgd2}
\frac{1}{T}\sum^T_{t=1} \E P(\w^{t+1})- P(\w^*)\le  \frac{1}{T}\left[\alpha\Bpsi(\w^*,\w^1) + \E \sum^T_{t=1}\frac{V_t}{\sigma (\alpha+\mu t)}\right],
\end{eqnarray}
where the variance is defined as $V_t =\V[ (np_{i_t}^t)^{-1} \nabla\phi_{i_t}(\wt)]=\E \| (np_{i_t}^t)^{-1} \nabla \phi_{i_t}(\wt)- \nabla f(\wt)\|_*^2$, and the expectation is take with the distribution $\p^t$.
%In addition, if we use $\alpha$-suffix averaging, defined as the average of the last $\alpha T$ iterates (where $\alpha\in(0,1)$ is a constant, and $\alpha T$ is assumed to be an integer): $\bar{\w}^{T+1}_\alpha=\frac{1}{\alpha T}\sum^T_{t=T_{\beta}}\w^{t+1}$, where $T_{\beta}=\beta T+1$ and $\beta=1-\alpha$.
\end{thm}
\begin{proof}
Firstly it is easy to check $\eta_t\in(0, \gamma]$. Because the functions $\psi$, $f$, $r$ satisfy the the assumptions in Lemma~\ref{lemma:primal-gap-t}, we have
\begin{eqnarray*}
\E[P(\w^{t+1})-P(\w^*)] \le\frac{1}{\eta_t}\E[\Bpsi(\w^*,\wt) -\Bpsi(\w^*,\w^{t+1})] -\mu\E\Bpsi(\w^*,\wt) +  \frac{\eta_t}{\sigma}\E\V\left[ (np_{i_t}^t)^{-1} \nabla\phi_{i_t}(\wt)\right].
\end{eqnarray*}
Summing the above inequality over $t=1,\ldots,T$, and using $\eta_t=1/(\alpha +\mu t)$  we get
\begin{eqnarray*}
&& \sum^T_{t=1} E P(\w^{t+1})-\sum^T_{t=1}P(\w^*) \\
&&\le  \sum^T_{t=1}(\alpha+\mu t)\E\left[\Bpsi(\w^*,\wt) -\Bpsi(\w^*,\w^{t+1})\right] -\mu \sum^T_{t=1}\E\Bpsi(\w^*,\wt) +   \E\sum^T_{t=1}\frac{V_t}{\sigma(\alpha+\mu t)}\\
&&= \alpha\Bpsi(\w^*,\w^1) - (\alpha+\mu T)\Bpsi(\w^*,\w^{T+1}) + \E \sum^T_{t=1}\frac{V_t}{\sigma(\alpha+\mu t)}\le  \alpha\Bpsi(\w^*,\w^1) + \E \sum^T_{t=1}\frac{V_t}{\sigma(\alpha+\mu t)}.
\end{eqnarray*}
Dividing both sides of the above inequality by $T$ concludes the proof.
\end{proof}

\begin{cor}
Under the same assumptions in the Theorem~\ref{thm:psgd}, if we further assume $\phi_i(\w)$ is $(1/\gamma_i)$-smooth, $\|\wt\|\le R$ for any $t$, and the distribution is set as $p_i^t=\frac{R/\gamma_i}{\sum^n_{j=1} R/\gamma_j}$, then the following inequality holds for any $T\ge 1$,
\begin{eqnarray*}
\frac{1}{T}\sum^T_{t=1} \E P(\w^{t+1})- P(\w^*)&\le& \frac{1}{T}\left[\alpha\Bpsi(\w^*,\w^1) + \frac{(\sum^n_{i=1} R/\gamma_i)^2}{\sigma \mu n^2}\left(\frac{\mu}{\alpha+\mu}+ \ln(\alpha+\mu T)-\ln(\alpha+\mu)\right)\right]\nonumber\\
&=& O\left(\frac{(\sum^n_{i=1}R/\gamma_i)^2}{\sigma \mu n^2}\frac{\ln (\alpha +\mu T)}{T}\right).
\end{eqnarray*}

In addition, if $\mu=0$, the above bound is invalid, however if $\eta_t$ is set as $\sqrt{\sigma\Bpsi(\w^*,\w^1)}/(\sqrt{T}\frac{\sum^n_{i=1}R/\gamma_i}{n})$, we can prove the following inequality for any $T\ge 1$,
\begin{eqnarray*}
 \frac{1}{T}\sum^T_{t=1} \E P(\w^{t+1})- P(\w^*)\le 2\sqrt{\frac{\Bpsi(\w^*,\w^1)}{\sigma}} \frac{\sum^n_{i=1}R/\gamma_i}{n}\frac{1}{\sqrt{T}}.
\end{eqnarray*}
\end{cor}

{\bf Remark:} If $\psi(\w)=\frac{1}{2}\|\w\|^2_2$ and $r(\w)=0$, then $\B_\psi(\u,\v)=\frac{1}{2}\|\u-\v\|^2_2$, and the proposed algorithm becomes SGD with importance sampling. Under these assumptions, it is achievable to get rid of the $\ln T$ factor in the convergence bound, when the objective function is strongly convex. However, we will not provide the details for concision. For more general Bregman divergence, it is difficult to remove this $\ln T$ factor, because many properties of $\frac{1}{2}\|\u-\v\|^2_2$ are not satisfied by the Bregman divergence, such as symmetry.

In addition, it is easy to derive high probability bound using existing work, such as the Theorem 8 in the~\cite{DBLP:conf/colt/DuchiSST10}. In this theorem, the high probability bound depends on the variance of the stochastic gradient, so our sampling strategy can improve this bound, since we are minimizing the variance. However, we do not explicitly provide the resulting bounds, because the consequence is relatively straightforward.

\begin{proof}
Firstly, the fact $\phi_i(\w)$ is $(1/\gamma_i)$-smooth, and $\|\wt\|\le R$ for any $t$ implies $\|\nabla \phi_i(\wt)\|_*\le R/\gamma_i$. Using this result and the distribution $p_i^t=\frac{R/\gamma_i}{\sum^n_{j=1}R/\gamma_j}$, we can get
\begin{eqnarray*}
V_t=\E \| (np_{i_t}^t)^{-1} \nabla \phi_{i_t}(\wt)- \nabla f(\wt)\|_*^2\le \E \| (np_{i_t}^t)^{-1} \nabla \phi_{i_t}(\wt)\|_*^2=\frac{1}{n^2}\sum^n_{i=1}\frac{1}{p_i}\|\nabla \phi_i(\wt)\|_*^2\le \left(\frac{\sum^n_{i=1} R/\gamma_i}{n}\right)^2,
\end{eqnarray*}
where the first inequality is due to $\E\|\z-\E\z\|^2=\E\|\z\|^2-\|\E\z\|^2$. Using the above inequality gives
\begin{eqnarray*}
\sum^T_{t=1}\frac{V_t}{\sigma (\alpha+\mu t)}&\le& \left(\frac{\sum^n_{i=1} R/\gamma_i}{n}\right)^2 \sum^T_{t=1}\frac{1}{\sigma (\alpha+\mu t)}\\
&\le& \left(\frac{\sum^n_{i=1} R/\gamma_i}{n}\right)^2\frac{1}{\sigma}\left[\frac{1}{\alpha+\mu}+\int^T_{t=1}\frac{1}{ \alpha+\mu t}\right]\\
&\le& \left(\frac{\sum^n_{i=1} R/\gamma_i}{n}\right)^2\frac{1}{\sigma \mu}\left[\frac{\mu}{\alpha+\mu}+ \ln(\alpha+\mu T)-\ln(\alpha+\mu)\right]
\end{eqnarray*}
Plugging the above inequality into the inequality~(\ref{eqn:psgd2}) concludes the proof of the first part.

To prove the second part, we can plug the bound on $V_t$ and the equality $\eta_t=\sqrt{\sigma\Bpsi(\w^*,\w^1)}/(\sqrt{T}\frac{\sum^n_{i=1}R/\gamma_i}{n})$ into the inequality~(\ref{eqn:psgd2}).
\end{proof}

\textbf{Remark.} If the uniform distribution is adopted, it is easy to observe that $V_t$ is bounded by $\frac{\sum^n_{i=1}(R/\gamma_i)^2}{n}$, and the Theorem 1 will results in  $\frac{1}{T}\sum^T_{t=1}\E P(\w^{t+1})- P(\w^*)\le  O\left(\frac{\sum^n_{i=1}(R/\gamma_i)^2}{\sigma \mu n}\frac{\ln (\alpha +\mu T)}{T}\right)$ for strongly convex $f$, and $\frac{1}{T}\sum^T_{t=1}\E P(\w^{t+1})- P(\w^*)\le  2\sqrt{\frac{\Bpsi(\w^*,\w^1)}{\sigma} \frac{\sum^n_{i=1}(R/\gamma_i)^2}{n}}\frac{1}{\sqrt{T}}$ for general convex $f$. However, according to Cauchy-Schwarz inequality,
\begin{eqnarray*}
\frac{\sum^n_{i=1}(R/\gamma_i)^2}{n}/\left( \frac{\sum^n_{i=1}R/\gamma_i}{n}\right)^2 =\frac{n\sum^n_{i=1}(R/\gamma_i)^2}{(\sum^n_{i=1}R/\gamma_i)^2} \ge 1,
\end{eqnarray*}
implies importance sampling does improve the convergence rate, especially when $\frac{(\sum^n_{i=1}R/\gamma_i)^2}{\sum^n_{i=1}(R/\gamma_i)^2} \ll n $.

\begin{thm}\label{thm:comid}
Let $\wt$ be generated by the proposed algorithm. Assume that $\psi(\cdot)$ is $\sigma$-strongly convex with respect to a norm $\|\cdot\|$,  $f$ is convex, and $r(\w)$ is $1$-strongly convex, if $\eta_t$ is set as $1/(\lambda t)$ for all $t$, the following inequality holds for any $T\ge 1$,
\begin{eqnarray}\label{eqn:comid}
\frac{1}{T}\sum^T_{t=1} \E P(\wt)- P(\w^*)\le \frac{1}{T}\left[\lambda\Bpsi(\w^*,\w^1)+\frac{1}{\lambda\sigma}\sum^T_{t=1}\frac{1}{t}\E\|\frac{\nabla \phi_{i_t}(\wt)}{np_{i_t}^t}\|_*^2\right],
\end{eqnarray}
where  the expectation is take with the distribution $\p^t$.
\end{thm}
\begin{proof}
The fact $r(\w)$ is $1$-strongly convex implies that $\lambda r(\w)$ is $\lambda$-strongly convex. Then, all the assumptions in the Corollary~6 of~\cite{DBLP:conf/colt/DuchiSST10} are satisfied, so  we have the following inequality,
\begin{eqnarray*}
\sum^T_{t=1}[\frac{1}{n p^t_{i_t}}\phi_{i_t}(\wt)+\lambda r(\w^{t+1})-\frac{1}{n p^t_{i_t}}\phi_{i_t}(\w^*)-\lambda r(\w^*)]\le \lambda \Bpsi(\w^*,\w^1) +\frac{1}{\lambda\sigma}\sum^T_{t=1}\frac{1}{t}\|\frac{1}{n p^t_{i_t}}\nabla \phi_{i_t}(\wt)\|_*^2
\end{eqnarray*}
which is actually the same with the last display in the page 5 of~\cite{DBLP:conf/colt/DuchiSST10}. Taking expectation on both sides of the above inequality and using $r(\w^1)=0$ concludes the proof.
\end{proof}

We will use the above Theorem to derive two logarithmic convergence bounds.

\begin{cor}
Under the same assumptions in the Theorem~\ref{thm:comid}, if we further assume $\phi_i(\w)$ is $(1/\gamma_i)$-smooth, $\|\wt\|\le R$ for any $t$, and the distribution is set as $p_i^t=\frac{R/\gamma_i}{\sum^n_{j=1} R/\gamma_j}$, then the following inequality holds for any $T\ge 1$,
\begin{eqnarray*}
 \frac{1}{T}\sum^T_{t=1} \E P(\wt)- P(\w^*)\le \frac{1}{T}\left[ \lambda \Bpsi(\w^*,\w^1) +\frac{\left(\sum^n_{i=1}R/\gamma_i\right)^2}{\lambda\sigma n^2} (\ln T + 1)\right] = O\left(\frac{\left(\sum^n_{i=1}R/\gamma_i\right)^2}{\lambda\sigma n^2} \frac{\ln T}{T} \right).
\end{eqnarray*}
Under the same assumptions in the theorem~\ref{thm:comid}, if $\phi_i(\w)$ is $L_i$-Lipschitz, and the distribution is set as $p_i=L_i/\sum^n_{j=1}L_j$, $\forall i$.
\begin{eqnarray*}
 \frac{1}{T}\sum^T_{t=1} \E P(\wt)- P(\w^*)\le\frac{1}{T}\left[ \lambda \Bpsi(\w^*,\w^1) +\frac{\left(\sum^n_{i=1}L_i\right)^2}{\lambda\sigma n^2} (\ln T + 1)\right] = O\left(\frac{\left(\sum^n_{i=1}L_i\right)^2}{\lambda\sigma n^2} \frac{\ln T}{T} \right).
\end{eqnarray*}
\end{cor}

\begin{proof}
If $\|\phi_i(\wt)\|_* \le G_i$ and $p^t_i$ is set as $\frac{G_i}{\sum^n_{j=1} G_j}$ for any $i$,  we have
\begin{eqnarray*}\label{eqn:expectation-squared-variable}
\E\|\frac{\nabla \phi_{i_t}(\wt)}{n p^t_{i_t}}\|_*^2 =\frac{1}{n^2}\sum^n_{i=1}\frac{\|\nabla\phi_i(\wt)\|_*^2}{p_i}\le\frac{(\sum^n_{i=1}G_i)^2}{n^2}.
\end{eqnarray*}
Since the same assumptions in the Theorem~\ref{thm:comid}  hold, plugging the above inequality into the inequality~(\ref{eqn:comid}), and using the fact $\sum^T_{t=1}\le 1 + \ln T$ results in
\begin{eqnarray}\label{eqn:comid-bound-G}
 \frac{1}{T}\sum^T_{t=1} \E P(\wt)- P(\w^*)\le\frac{1}{T}\left[ \lambda \Bpsi(\w^*,\w^1) +\frac{\left(\sum^n_{i=1}G_i\right)^2}{\lambda\sigma n^2} (\ln T + 1)\right] = O\left(\frac{\left(\sum^n_{i=1}G_i\right)^2}{\lambda\sigma n^2} \frac{\ln T}{T} \right).
\end{eqnarray}

When $\phi_i(\w)$ is $(1/\gamma_i)$-smooth, $\|\wt\|\le R$ for any $t$, we have $\|\nabla \phi_i(\wt) \|_* \le R/\gamma_i$. Plugging $G_i=R/\gamma_i$ into the inequality~(\ref{eqn:comid-bound-G}) concludes the proof of the first part.

When $\phi_i(\w)$ is $L_i$-Lipschitz, we have $\|\nabla\phi_i(\wt)\|_*\le L_i$. Plugging $G_i=L_i$ into the inequality~(\ref{eqn:comid-bound-G}) concludes the proof of the second part.
\end{proof}

{\bf Remark:} If the uniform distribution is adopted, it is easy to observe that $V_t$ is bounded by $\frac{\sum^n_{i=1}(R/\gamma_i)^2}{n}$ for smooth $\phi_i(\cdot)$, and bounded by $\frac{\sum^n_{i=1}(L_i)^2}{n}$ for Lipschitz $\phi_i(\cdot)$. So the Theorem 2 will results in  $\frac{1}{T}\sum^T_{t=1}\E P(\wt)- P(\w^*)\le   O\left(\frac{\sum^n_{i=1}(R/\gamma_i)^2}{\lambda\sigma n} \frac{\ln T}{T} \right)$ for smooth $\phi_i$, and $\frac{1}{T}\sum^T_{t=1}\E P(\wt)- P(\w^*)\le   O\left(\frac{\sum^n_{i=1}(L_i)^2}{\lambda\sigma n} \frac{\ln T}{T} \right)$ for Lipschitz $\phi_i$. However, according to Cauchy-Schwarz inequality,
\begin{eqnarray*}
\frac{\sum^n_{i=1}(R/\gamma_i)^2}{n}/\left( \frac{\sum^n_{i=1}R/\gamma_i}{n}\right)^2 =\frac{n\sum^n_{i=1}(R/\gamma_i)^2}{(\sum^n_{i=1}R/\gamma_i)^2} \ge 1,\quad \frac{\sum^n_{i=1}L_i^2}{n}/(\frac{\sum^n_{i=1} L_i}{n})^2 = \frac{n\sum^n_{i=1}L_i^2}{(\sum^n_{i=1}L_i)^2}\ge 1
\end{eqnarray*}
implies importance sampling does improve the convergence rate, especially when $\frac{(\sum^n_{i=1}R/\gamma_i)^2}{\sum^n_{i=1}(R/\gamma_i)^2} \ll n $, and $\frac{(\sum^n_{i=1}L_i)^2}{\sum^n_{i=1}(L_i)^2} \ll n $.

\begin{thm}\label{thm:comid-convex}
Let $\wt$ be generated by the proposed algorithm. Assume that $\psi(\cdot)$ is $\sigma$-strongly convex with respect to a norm $\|\cdot\|$,  $f$  and $r(\w)$ are convex, if $\eta_t =\eta$, the following inequality holds for any $T\ge 1$,
\begin{eqnarray}\label{eqn:comid-convex}
\frac{1}{T}\sum^T_{t=1} \E P(\wt)- P(\w^*)\le \frac{1}{T}\left[\frac{1}{\eta} \Bpsi(\w^*,\w^1)+\frac{\eta}{2\sigma}\sum^T_{t=1}\E\|\frac{\nabla \phi_{i_t}(\wt)}{np_{i_t}^t}\|_*^2\right],
\end{eqnarray}
where  the expectation is take with the distribution $\p^t$.
\end{thm}
\begin{proof}
Given the above conditions, all the assumptions of the Lemma 1 in the page 3 of~\cite{DBLP:conf/colt/DuchiSST10} are satisfied. So the following inequality holds for any $T\ge 1$,
\begin{eqnarray*}
\sum^T_{t=1}[\frac{1}{n p^t_{i_t}}\phi_{i_t}(\wt)+\lambda r(\wt)-\frac{1}{n p^t_{i_t}}\phi_{i_t}(\w^*)-\lambda r(\w^*)]\le \frac{1}{\eta} \Bpsi(\w^*,\w^1) +\lambda r(\w^1)+\frac{\eta}{2 \sigma}\sum^T_{t=1}\|\frac{1}{n p^t_{i_t}}\nabla \phi_{i_t}(\wt)\|_*^2,
\end{eqnarray*}
which is actually the same with the inequality in the Theorem 2 in the page 4 of~\cite{DBLP:conf/colt/DuchiSST10}. Taking expectation on both sides of the above inequality  and using $r(\w^1)=0$ concludes the proof.
\end{proof}

\begin{cor}
Under the same assumptions in the Theorem~\ref{thm:comid-convex}, if we further assume $\phi_i(\w)$ is $(1/\gamma_i)$-smooth, $\|\wt\|\le R$ for any $t$, and the distribution is set as $p_i^t=\frac{R/\gamma_i}{\sum^n_{j=1} R/\gamma_j}$, then when $\eta_t$ is set as $\sqrt{2\sigma\Bpsi(\w^*,\w^1)}/(\frac{\sum^n_{i=1}R/\gamma_i}{n}\sqrt{T} )$, the following inequality holds for any $T\ge 1$,
\begin{eqnarray*}
\frac{1}{T}\sum^T_{t=1} \E P(\wt)- P(\w^*)\le \sqrt{   \Bpsi(\w^*,\w^1) \frac{2}{ \sigma}} (\frac{\sum^n_{i=1}R/\gamma_i}{n})\frac{1}{\sqrt{T}}.
\end{eqnarray*}
Under the same assumptions in the theorem~\ref{thm:comid-convex}, if $\phi_i(\w)$ is $L_i$-Lipschitz, and the distribution is set as $p_i=L_i/\sum^n_{j=1}L_j$, $\forall i$, then when $\eta_t$ is set as $\sqrt{2\sigma\Bpsi(\w^*,\w^1)}/(\frac{\sum^n_{i=1}L_i}{n}\sqrt{T} )$, the following inequality holds for any $T\ge 1$,
\begin{eqnarray*}
\frac{1}{T}\sum^T_{t=1} \E P(\wt)- P(\w^*) \le \sqrt{   \Bpsi(\w^*,\w^1) \frac{2}{ \sigma}} (\frac{\sum^n_{i=1}L_i}{n})\frac{1}{\sqrt{T}}.
\end{eqnarray*}
\end{cor}

\begin{proof}
Under the same assumptions in the Theorem~\ref{thm:comid-convex}, if $\|\phi_i(\wt)\|_*\le G_i$ and $p_i^t$ is set as $\frac{G_i}{\sum^n_{j=1}G_j}$ for any $i$, then we have
\begin{eqnarray*}
\E\|\frac{\nabla \phi_{i_t}(\wt)}{n p^t_{i_t}}\|_*^2 \le\frac{(\sum^n_{i=1}G_i)^2}{n^2}.
\end{eqnarray*}
Plugging the above inequality and $\eta_t = \sqrt{2\sigma\Bpsi(\w^*,\w^1)}/(\frac{\sum^n_{i=1}G_i}{n}\sqrt{T} )$, into the inequality~(\ref{eqn:comid-convex}) gives
\begin{eqnarray}\label{eqn:comid-convex-bound-G}
\frac{1}{T}\sum^T_{t=1} \E P(\wt)- P(\w^*)\le \sqrt{   \Bpsi(\w^*,\w^1) \frac{2}{ \sigma}} (\frac{\sum^n_{i=1}G_i}{n})\frac{1}{\sqrt{T}}.
\end{eqnarray}

When $\phi_i(\w)$ is $(1/\gamma_i)$-smooth, $\|\wt\|\le R$ for any $t$, we have $\|\nabla \phi_i(\wt) \|_* \le R/\gamma_i$. Plugging $G_i=R/\gamma_i$ into the inequality~(\ref{eqn:comid-convex-bound-G}) concludes the proof of the first part.

When $\phi_i(\w)$ is $L_i$-Lipschitz, we have $\|\nabla\phi_i(\wt)\|_*\le L_i$. Plugging $G_i=L_i$ into the inequality~(\ref{eqn:comid-convex-bound-G}) concludes the proof of the second part.
\end{proof}

{\bf Remark:} If the uniform distribution is adopted, it is easy to observe that $V_t$ is bounded by $\frac{\sum^n_{i=1}(R/\gamma_i)^2}{n}$ for smooth $\phi_i(\cdot)$, and bounded by $\frac{\sum^n_{i=1}(L_i)^2}{n}$ for Lipschitz $\phi_i(\cdot)$. So the Theorem 2 will results in  $\frac{1}{T}\sum^T_{t=1}\E P(\wt)- P(\w^*)\le \sqrt{    \frac{2\Bpsi(\w^*,\w^1)\sum^n_{i=1}(R/\gamma_i)^2}{\sigma nT}}$ for smooth $\phi_i$, and  $\frac{1}{T}\sum^T_{t=1}\E P(\wt)- P(\w^*)\le \sqrt{    \frac{2\Bpsi(\w^*,\w^1)\sum^n_{i=1}(L_i)^2}{\sigma nT}}$  for Lipschitz $\phi_i$. However, according to Cauchy-Schwarz inequality,
\begin{eqnarray*}
\frac{\sum^n_{i=1}(R/\gamma_i)^2}{n}/\left( \frac{\sum^n_{i=1}R/\gamma_i}{n}\right)^2 =\frac{n\sum^n_{i=1}(R/\gamma_i)^2}{(\sum^n_{i=1}R/\gamma_i)^2} \ge 1,\quad \frac{\sum^n_{i=1}L_i^2}{n}/(\frac{\sum^n_{i=1} L_i}{n})^2 = \frac{n\sum^n_{i=1}L_i^2}{(\sum^n_{i=1}L_i)^2}\ge 1
\end{eqnarray*}
implies importance sampling does improve the convergence rate, especially when $\frac{(\sum^n_{i=1}R/\gamma_i)^2}{\sum^n_{i=1}(R/\gamma_i)^2} \ll n $, and $\frac{(\sum^n_{i=1}L_i)^2}{\sum^n_{i=1}(L_i)^2} \ll n $.

%============================================================================================
%
%  Proximal  Stochastic Dual Coordinate Ascent with Importance Sampling
%
%============================================================================================

\subsection{Proximal  Stochastic Dual Coordinate Ascent with Importance Sampling}
In this section, we study the Proximal Stochastic Dual Coordinate Ascent method (prox-SDCA) with importance sampling.
Prox-SDCA works with the following dual problem of~\eqref{eqn:primal-objective}:
\begin{eqnarray}
\max_{\theta} D(\theta) := \frac{1}{n}\sum^n_{i=1}-\phi_i^*(-\theta_i)-\lambda r^*(\frac{1}{\lambda n}\sum^n_{i=1}\theta_i).
\end{eqnarray}

We assume that $r^*(\cdot)$ is continuous differentiable; the relationship between primal variable $\w$ and dual variable is
\begin{eqnarray*}
\w =\nabla r^*\left(\v(\theta) \right),\quad \v(\theta)=\frac{1}{\lambda n}\sum^n_{i=1} \theta_i.
\end{eqnarray*}

We also assume that $r(\w)$ is 1-strongly convex with respect to a norm $\|\cdot\|_{P'}$, i.e.,
\begin{eqnarray*}
r(\w + \Delta \w)\ge r(\w) + \nabla r(\w)^\top\Delta\w + \frac{1}{2}\|\Delta\w\|^2_{P'},
\end{eqnarray*}
which means that $r^*(\w)$ is 1-smooth with respect to its dual norm $\|\cdot\|_{D'}$. Namely,
\begin{eqnarray*}
r^*(\v+\Delta \v)\le h(\v; \Delta \v),
\end{eqnarray*}
where
\begin{eqnarray*}
h(\v;\Delta\v):= r^*(\v)+\nabla r^*(\v)^\top \Delta\v + \frac{1}{2}\|\Delta\v\|^2_{D'}.
\end{eqnarray*}

At the $t$-th step, traditional Proximal Stochastic Dual Coordinate Ascent (prox-SDCA)  will uniformly randomly pick $i\in\{1,\ldots,n\}$, and update the dual variable $\theta^{t-1}_i$ as follows:
\begin{eqnarray*}
\theta^t_{i}=\theta^{t-1}_i + \Delta \theta^{t-1}_i,
\end{eqnarray*}
where
\begin{eqnarray}\label{eqn:update-psdca}
\Delta \theta^{t-1}_i &=& \arg\max_{\Delta \theta_i}\left[-\frac{1}{n}\phi^*_i(-(\theta_i^{t-1}+\Delta \theta_i))-\lambda\left(\frac{1}{\lambda n}\nabla r^*(\v^{t-1})^\top\Delta\theta_i+\frac{1}{2}\|\frac{1}{\lambda n}\Delta \theta_i\|^2_{D'}\right)\right] \nonumber\\
&=& \arg\max_{\Delta \theta_i}\left[-\phi^*_i(-(\theta_i^{t-1}+\Delta \theta_i)) - (\w^{t-1})^\top\Delta \theta_i - \frac{1}{2\lambda n}\|\Delta \theta_i\|^2_{D'} )\right],
\end{eqnarray}
where $\v^{t-1}=\frac{1}{\lambda n}\sum^n_{i=1}\theta^{t-1}_i$, which is equivalent to maximizing a lower bound of the following problem:
\begin{eqnarray*}
\Delta\theta_i^{t-1}=\arg\max_{\Delta \theta_i}\left[-\frac{1}{n}\phi^*_i(-(\theta_i^{t-1}+\Delta \theta_i))-\lambda r^*(\v^{t-1}+ \frac{1}{\lambda n}\Delta \theta_i) \right].
\end{eqnarray*}

However, the optimization~\eqref{eqn:update-psdca} may not have a closed form solution, and in prox-SDCA we may adopt other update rules $\Delta \theta_i = s(\u- \theta^{t-1}_i)$ for an appropriately chosen step size parameter $s>0$ and any vector $\u\in \R^d$ such that $-\u\in\partial\phi_i(\w^{t-1})$. When  $r(\w)=\frac{1}{2}\|\w\|^2$, the proximal SDCA method is known as SDCA.

Now we will study prox-SDCA with importance sampling, which is to allow the algorithm to randomly pick $i$ according to probability $p_i$, which is the $i$-th element of $\p\in\R^n_+$, $\sum p_i=1$. Once we pick the coordinate $i$, $\theta_i$ is updated as traditional prox-SDCA. The main question we are interested in here is which $\p=(p_1,\ldots,p_n)^\top$ can optimally accelerate the convergence rate of prox-SDCA. To answer this question, we will introduce a lemma which will state the relationship between $\p$ and the convergence rate of prox-SDCA with importance sampling.

\begin{lemma}\label{lemma:dual-ascent-psdca}
Given a distribution $\p$, if assume $\phi_i$ is $(1/\gamma_i)$-smooth with norm $\|\cdot\|_P$, then for any iteration $t$ and any $s$ such that $s_i= s/(p_i n)\in [0,1],\quad \forall i$, we have
\begin{eqnarray}\label{eqn:dual-ascent-ipsdca}
\E[D(\theta^t)-D(\theta^{t-1})]\ge \frac{s}{n}\E[P(\w^{t-1})-D(\theta^{t-1})]-\frac{s}{2\lambda n^2}G^t,
\end{eqnarray}
where
\begin{eqnarray*}
G^t = \frac{1}{n}\sum^n_{i=1}(s_i R^2-\gamma_i(1-s_i)\lambda n)\E\|\u^{t-1}_i-\theta^{t-1}_i\|^2_D,
\end{eqnarray*}
$R=\sup_{\u\not=0}\|\u\|_{D'}/\|\u\|_D$, and $-\u^{t-1}_i\in \partial\phi_i(\w^{t-1})$.
\end{lemma}
\begin{proof}
Since only the $i$-th element of $\theta$ is updated, the improvement in the dual objective can written as
\begin{eqnarray}\label{eqn:stochastic-dual-ascent-ipsdca}
&&n[D(\theta^t)-D(\theta^{t-1})]\nonumber\\
&&=\left[-\phi^*_i(-\theta^t_i)-\lambda n r^*(\v^{t-1} +\frac{1}{\lambda n}\Delta \theta^{t-1}_i)\right]-\left[-\phi_i^*(-\theta^{t-1}_i)-\lambda n r^*(\v^{t-1})\right]\nonumber\\
&&\ge \underbrace{\left[-\phi^*_i(-\theta^t_i)-\lambda n h(\v^{t-1} ;\frac{1}{\lambda n}\Delta \theta^{t-1}_i)\right]}_{A_i}-\underbrace{\left[-\phi_i^*(-\theta^{t-1}_i)-\lambda n r^*(\v^{t-1})\right]}_{B_i}.
\end{eqnarray}
By the definition of the update, for all $s_i\in[0,1]$  we have
\begin{eqnarray}\label{eqn:Ai}
A_i&=&\max_{\Delta\theta_i}-\phi_i^*(-(\theta^{t-1}_i+\Delta\theta_i))-\lambda n h(\v^{t-1} ;\frac{1}{\lambda n}\Delta \theta_i)\nonumber\\
&\ge& -\phi_i^*(-(\theta_i^{t-1}+s_i(\u_i^{t-1}-\theta_i^{t-1})))-\lambda n h(\v^{t-1}; \frac{s_i}{\lambda n}(\u_i^{t-1}-\theta_i^{t-1})) .
\end{eqnarray}
From now on, we drop the superscript $(t-1)$ to simplify our discussion. Because $\phi_i$ is $(1/\gamma_i)$-smooth with $\|\cdot\|_P$, $\phi_i^*$ is $\gamma_i$-strongly convex with $\|\cdot\|_D$, and we have that
\begin{eqnarray*}
\phi_i^*(-(\theta_i+s_i(\u_i-\theta_i)))=\phi_i^*(s_i(-\u_i)+(1-s_i)(-\theta_i))\le s_i\phi_i^*(-\u_i)+(1-s_i)\phi_i^*(-\theta_i) - \frac{\gamma_i}{2}s_i(1-s_i)\|\u_i-\theta_i\|^2_D .
\end{eqnarray*}
Combining the above two inequalities, we obtain
\begin{eqnarray*}
&&\hspace{-0.3in}A_i\\
&&\hspace{-0.3in}\ge -s_i\phi_i^*(-\u_i) - (1-s_i)\phi_i^*(-\theta_i)+\frac{\gamma_i}{2}s_i(1-s_i)\|\u_i-\theta_i\|^2_D-\lambda n h(\v; \frac{s_i}{\lambda n}(\u_i-\theta_i))\\
&&\hspace{-0.3in}= -s_i\phi_i^*(-\u_i) - (1-s_i)\phi_i^*(-\theta_i)+\frac{\gamma_i}{2}s_i(1-s_i)\|\u_i-\theta_i\|^2_D-\lambda n r^*(\v)-s_i\w^\top(\u_i-\theta_i)-\frac{s_i^2}{2\lambda n}\|\u_i-\theta_i\|^2_{D'}\\
&&\hspace{-0.3in}\ge -s_i\phi_i^*(-\u_i) - (1-s_i)\phi_i^*(-\theta_i)+\frac{\gamma_i}{2}s_i(1-s_i)\|\u_i-\theta_i\|^2_D-\lambda n r^*(\v)-s_i\w^\top(\u_i-\theta_i)-\frac{s_i^2}{2\lambda n}R^2\|\u_i-\theta_i\|^2_{D}\\\\
&&\hspace{-0.3in}=\underbrace{-s_i(\phi_i^*(-\u_i)+\u_i^\top\w)}_{s_i\phi_i(\w)}+\underbrace{(-\phi_i^*(-\theta_i)-\lambda nr^*(\v))}_{B_i}+\frac{s_i}{2}(\gamma_i(1-s_i)-\frac{s_i R^2}{\lambda n})\|\u_i-\theta_i\|^2_D + s_i (\phi_i^*(-\theta_i)+\theta_i^\top\w),
\end{eqnarray*}
where we used $-\u^{t-1}_i\in \partial\phi_i(\w^{t-1})$ which yields $\phi_i^*(-\u_i)=-\u_i^\top\w - \phi_i(\w_i)$. Therefore
\begin{eqnarray}\label{eqn:bound-A-minus-B}
A_i-B_i\ge s_i\Big[\phi_i(\w) +\phi_i^*(-\theta_i)+\theta_i^\top \w +\left(\frac{\gamma_i(1-s_i)}{2}-\frac{s_i R^2}{2\lambda n}\right)\|\u_i-\theta_i\|^2_D \Big].
\end{eqnarray}
Therefore, if we take expectation of inequality~\eqref{eqn:bound-A-minus-B}  with respect to the choice of $i$ and use the fact $s_i = s/(p_i n)$, we obtain that
\begin{eqnarray*}
\frac{1}{s}\E_t [A_i-B_i] = \frac{1}{n}\sum^n_{i=1}\frac{1}{s_i}[A_i-B_i]\ge\frac{1}{n}\sum^n_{i=1}\Big[\phi_i(\w)+\phi_i^*(-\theta_i)+\theta_i^\top \w +\left(\frac{\gamma_i(1-s_i)}{2}-\frac{s_i R^2}{2\lambda n}\right)\|\u_i-\theta_i\|^2_D\Big].
\end{eqnarray*}
Next note that with $\w=\nabla r^*(\v)$, we have $r(\w)+r^*(\v)=\w^\top\v$. Therefore
\begin{eqnarray*}
P(\w)-D(\theta) &=& \frac{1}{n}\sum^n_{i=1}\phi_i(\w) + \lambda r(\w) - \left(-\frac{1}{n}\sum^n_{i=1}\phi_i^*(-\theta_i)-\lambda r^*(\v)\right)\\
&=&\frac{1}{n}\sum^n_{i=1}\phi_i(\w)+\frac{1}{n}\sum^n_{i=1}\phi_i^*(-\theta_i)+\lambda\w^\top\v\\
&=&\frac{1}{n}\sum^n_{i=1}(\phi_i(\w)+\phi_i^*(-\theta_i)+\theta_i^\top\w).
\end{eqnarray*}
Taking expectation of inequality~\eqref{eqn:stochastic-dual-ascent-ipsdca} and plugging the above two equations give
\begin{eqnarray*}
\frac{n}{s}\E[D(\theta^t)-D(\theta^{t-1})]\ge \frac{1}{s}\E[A_i-B_i]\ge \E[P(\w^{t-1})-D(\theta^{t-1})]- \frac{G^t}{2\lambda n},
\end{eqnarray*}
where the equality $G^t = \frac{1}{n}\sum^n_{i=1}(s_i R^2-\gamma_i(1-s_i)\lambda n)\E\|\u^{t-1}_i-\theta^{t-1}_i\|^2_D$ is used. Multiplying both sides by $s/n$ concludes the proof.
\end{proof}
For many interesting cases, it is easy to estimate $R=\sup_{\u\not=0}\|\u\|_{D'}/\|\u\|_D$. For example, if $p>r>0$, then $\|\w\|_p\le\|\w\|_r\le d^{(1/r-1/p)}\|\w\|_p$ for any $\w\in\R^d$.

\subsubsection{Algorithm}
According to Lemma~\ref{lemma:dual-ascent-psdca}, to maximize the dual ascent for the $t$-th update, we should choose $s$ and $\p$ as the solution of the following optimization
\begin{eqnarray*}
\max_{s/(p_i n)\in[0,1], \p\in\R^n_+,\sum^n_{i=1}p_i=1} \frac{s}{n}\E[P(\w^{t-1})-D(\theta^{t-1})]-\frac{s}{n^2}\frac{G^t}{2\lambda}.
\end{eqnarray*}
However, because this optimization problem is difficult to solve, we choose to relax it as follows:
\begin{eqnarray*}
&&\max_{s/(p_i n)\in[0,1], \p\in\R^n_+, \sum^n_{i=1}p_i=1} \frac{s}{n}\E[P(\w^{t-1})-D(\theta^{t-1})]-\frac{s}{n^2}\frac{G^t}{2\lambda}\\
&&\ge\max_{s/(p_i n)\in[0,\frac{\lambda n \gamma_i}{R^2+\lambda n \gamma_i}], \p\in\R^n_+, \sum^n_{i=1}p_i=1} \frac{s}{n}\E[P(\w^{t-1})-D(\theta^{t-1})]-\frac{s}{n^2}\frac{G^t}{2\lambda}\\
&&\ge \max_{s/(p_i n)\in[0,\frac{\lambda n \gamma_i}{R^2+\lambda n \gamma_i}], \p\in\R^n_+, \sum^n_{i=1}p_i=1} \frac{s}{n}\E[P(\w^{t-1})-D(\theta^{t-1})].
\end{eqnarray*}
where the last inequality used $G^t = \frac{1}{n}\sum^n_{i=1}(s_i R^2-\gamma_i(1-s_i)\lambda n)\E\|\u^{t-1}_i-\theta^{t-1}_i\|^2_D\le 0$, since $s_i=s/(p_i n)\le \frac{\lambda n \gamma_i}{R^2+\lambda n \gamma_i}$. To optimize the final relaxation, we have the following proposition
\begin{prop}\label{prop:distribution-ipsdca}
The solution for the optimization
\begin{eqnarray*}
\max_{s, \p}\  s\quad s.t.\ s/(p_i n)\in[0, \frac{\lambda n\gamma_i}{R^2+ \lambda n\gamma_i}],\ p_i\ge 0,\  \sum^n_{i=1}p_i =1,
\end{eqnarray*}
is
\begin{eqnarray}\label{eqn:distribution-ipsdca}
s=\frac{n}{n+\sum^n_{i=1}\frac{R^2}{\lambda n \gamma_i}},\ p_i = \frac{1+\frac{R^2}{\lambda n\gamma_i}}{n + \sum^n_{j=1}\frac{R^2}{\lambda n \gamma_j}}.
\end{eqnarray}
\end{prop}
We omit the proof of this proposition since it is simple. Given that $\phi_i$ is ($1 /\gamma_i$)-smooth, $\forall i\in\{1,\ldots,n\}$, the sampling distribution should be set as in~\eqref{eqn:distribution-ipsdca}. Although $s$ can be set as in~\eqref{eqn:distribution-ipsdca}, it can also be optimized by maximizing some terms in the analysis, such as $A_i$, or the right hand side of inequality~\eqref{eqn:Ai}, or inequality~\eqref{eqn:bound-A-minus-B}, which can all guarantee the dual ascent $\E[D(\theta^t)-D(\theta^{t-1})]$ is no worse the the one
obtained by setting $s$ as in \eqref{eqn:distribution-ipsdca}.

When $\gamma_i=0$, the above distribution in the equation~\eqref{eqn:distribution-ipsdca} is not valid. To solve this problem, we  combine the facts
\begin{eqnarray*}
 P(\w^{t-1})-D(\theta^{t-1})\ge D(\theta^*)-D(\theta^{t-1}):=\epsilon^{t-1}_D,
\end{eqnarray*}
where $\theta^*$ the optimal solution of the dual problem $\max_\theta D(\theta)$,
\[
D(\theta^t)-D(\theta^{t-1})=\epsilon^{t-1}_D - \epsilon^t_D,
\]
and the inequality~~\eqref{eqn:dual-ascent-ipsdca}, to obtain
\begin{eqnarray}
\E[\epsilon^t_D]\le(1-\frac{s}{n})\E[\epsilon^{t-1}_D]+\frac{s}{2\lambda n^2}G^t.
\end{eqnarray}

According to this inequality, although every $\gamma_i=0$, if we further assume every $\phi_i$ is $L_i$-Lipschitz, then
\begin{eqnarray}\label{eqn:Gt-bound}
G^t = \frac{1}{n}\sum^n_{i=1}(s_i R^2-\gamma_i(1-s_i)\lambda n)\E\|\u^{t-1}_i-\theta^{t-1}_i\|^2_D\le \frac{4R^2s}{n^2}\sum^n_{i=1}\frac{1}{p_i}L_i^2,
\end{eqnarray}
where we used $s_i=s/(n p_i)$, $\|\u^{t-1}_i\|\le L_i$ and $\|\theta^{t-1}_i\|\le L_i$, since $-\u^{t-1}_i,-\theta^{t-1}_i\in\partial \phi_i(\w^{t-1})$.

Combining the above two inequalities results in
 \begin{eqnarray}\label{eqn:dual-ascent-lipschitz}
\E[\epsilon^t_D]\le (1-\frac{s}{n})\E[\epsilon^{t-1}_D]+\frac{s}{2\lambda n^2}\frac{4R^2s}{n^2}\sum^n_{i=1}\frac{1}{p_i}L_i^2.
\end{eqnarray}

According to the above inequality,  to minimize the $t$-th duality gap, we should choose a proper distribution to optimize the following problem
\begin{eqnarray*}
\min_{\p\in\R^n_+,\sum^n_{i=1}p_i=1}\sum^n_{i=1}\frac{1}{p_i}L_i^2,
\end{eqnarray*}
 for which the optimal distribution is obviously
\[
p_i=\frac{L_i}{\sum^n_{j=1}L_j}.
\]
Because $s_i=s/(np_i)\in[0,1]$,  the above distribution furthermore indicates
\[
s\in\bigcap^n_{i=1}[0, np_i]=\left[0,  \frac{n L_{min}}{\sum^n_{j=1}L_j}\right]:=[0, \rho]
\]
where $L_{min}=\min\{L_1,L_2,\ldots,L_n\}$ and $\rho\le 1$.

 In summary, the prox-SDCA with importance sampling can be summarized as in the algorithm~\ref{alg:Iprox-SDCA}.

\begin{algorithm}[ht]
\caption{Proximal Stochastic Dual Coordinate Ascent with Importance Sampling (Iprox-SDCA)} \label{alg:Iprox-SDCA}
\begin{algorithmic}
\STATE {\bf Input}: $\lambda> 0$, $R=\sup_{\u\not=0}\|\u\|_{D'}/\|\u\|_{D}$, norms $\|\cdot\|_D$, $\|\cdot\|_{D'}$, $\gamma_1,\ldots,\gamma_n > 0$, or $L_1,\ldots,L_n\ge 0$.
\STATE {\bf Initialize}: $\theta^0_i=0$, $\w^0=\nabla r^*(0)$, $p_i= \frac{1+\frac{R^2}{\lambda n\gamma_i}}{n+\sum^n_{j=1}\frac{R^2}{\lambda n\gamma_j}}$, or $p_i=\frac{L_i}{\sum^n_{j=1}L_j}$, $\forall i\in\{1,\ldots,n\}$.
\FOR{$t=1,\ldots, T$}
\STATE Sample $i_t$ from $\{1,\ldots,n\}$ based on $\p$;
\STATE Calculate $\Delta \theta^{t-1}_{i_t}$ using any of the following options (or achieving larger dual objective than one of the options);
\STATE \textbf{Option I:}
\STATE $\Delta \theta^{t-1}_{i_t} =\arg\max_{\Delta \theta_{i_t}}\left[-\phi^*_{i_t}(-(\theta_{i_t}^{t-1}+\Delta \theta_{i_t})) - (\w^{t-1})^\top\Delta \theta_{i_t} - \frac{1}{2\lambda n}\|\Delta \theta_{i_t}\|^2_{D'} \right] $;
\STATE \textbf{Option II:}
\STATE Let $\u$ be s.t. $-\u\in \partial \phi_{i_t}(\w^{t-1})$;
\STATE Let $\z= \u-\theta^{t-1}_{i_t}$;
\STATE Let $s_{i_t}=\arg\max_{s\in[0,1]}\left[-\phi^*_{i_t}(-(\theta_{i_t}^{t-1}+s\z)) - s(\w^{t-1})^\top\z- \frac{s^2}{2\lambda n}\|\z\|^2_{D'} \right]$;
\STATE Set $\Delta\theta^{t-1}_{i_t} = s_{i_t}\z$;
\STATE \textbf{Option III:}
\STATE Same as Option II, but replace the definition of $s_{i_t}$ as follows:
\STATE Let $s_{i_t}=\frac{\phi_{i_t}(\w^{t-1})+\phi^*_{i_t}(-\theta^{t-1}_{i_t})+(\w^{t-1})^\top\theta^{t-1}_{i_t}+\frac{\gamma_{i_t}}{2}\|\z\|^2_D}{\|\z\|^2_D(\gamma_{i_t}+R^2/(\lambda n))}$;
\STATE \textbf{Option  IV (only for Lipschitz losses):}
\STATE Same as Option III, but replace $\|\z\|^2_D$ with $4L^2_{i_t}$;
\STATE \textbf{Option V (only for smooth losses):}
\STATE $\Delta \theta^{t-1}_{i_t}=\frac{n}{n+\sum^n_{i=1}\frac{R^2}{\lambda n \gamma_i}}\left(-\nabla\phi_{i_t}(\w^{t-1})-\theta^{t-1}_{i_t}\right)  $;
\STATE $\theta^{t}_{i_t}=\theta^{t-1}_{i_t}+\Delta\theta^{t-1}_{i_t}$;
\STATE $\v^t=\v^{t-1}+\frac{1}{\lambda n}\Delta\theta_{i_t}^{t-1}$;
\STATE $\wt = \nabla r^*(\v^t)$;
\ENDFOR
\end{algorithmic}
\end{algorithm}
{\bf Remark:} It is easy to check the first three options can do no worse than the Option IV for Lipschitz losses and Option V for smooth losses.
Specifically, option I is to optimize $A_i$ directly. Option II only requires choosing $\Delta\theta_i=s(\u^{t-1}_i-\theta^{t-1}_i)$ and then chooses $s$ to optimize a lower bound of $A_i$, i.e., the right hand side of inequality~\eqref{eqn:Ai}. Option III is similar with option II, and chooses $s$ to optimize the bound on the right hand side of~\eqref{eqn:bound-A-minus-B}. As a result, all the first three options can do no worse than choosing  the optimal $s$ for the Lemma~\ref{lemma:dual-ascent-psdca}. Option IV replace $\|\z\|^2_D$ with its upper bound $4L^2_{i_t}$, so that the inequality~\eqref{eqn:dual-ascent-lipschitz} is still valid. Option V is similar with options II and III, and chooses $s=\frac{n}{n+\sum^n_{i=1}\frac{R^2}{\lambda n \gamma_i}}$ as in the Proposition~\ref{prop:distribution-ipsdca}, so that $G^t\le 0$.

\subsubsection{Analysis}
In this subsection we analyze the convergence behavior of the proposed algorithm. Before presenting the theoretical results, we will make several assumptions without loss of generality:
\begin{itemize}
\item for the loss functions: $\phi_i(0)\le 1$, and $\phi_i(\w)\ge 0,\ \forall \w$, and
\item for the regularizer: $r(0)=0$ and $r(\w)\ge 0,\ \forall \w$.
\end{itemize}

Under the above assumptions,  we have the following theorem for the expected duality gap of $\E[P(\wt)-D(\theta^T)]$.
\begin{thm}
Assume  $\phi_i$ is ($1/\gamma_i$)-smooth $\forall i\in\{1,\ldots,n\}$ and set $p_i=(1+\frac{R^2}{\lambda n\gamma_i})/(n+\sum^n_{j=1}\frac{R^2}{\lambda n\gamma_j})$, for all $i\in\{1,\dots,n\}$. To obtain an expected duality gap of $\E[P(\wt)-D(\theta^T)]\le \epsilon_P$ for the proposed Proximal SDCA with importance sampling, it suffices to have a total number of iterations of
\begin{eqnarray*}
T \ge (n + \sum^n_{i=1}\frac{R^2}{\lambda n\gamma_i})\log\left((n+\sum^n_{i=1}\frac{R^2}{\lambda n\gamma_i})\frac{1}{\epsilon_P}\right).
\end{eqnarray*}
\end{thm}
\begin{proof}
Given the distribution $\p$ and step size $s$ in equation~\eqref{eqn:distribution-ipsdca}, since $\phi_i$ is ($1/\gamma_i$)-smooth, according to Lemma~\ref{lemma:dual-ascent-psdca}, we have
\begin{eqnarray*}
(s_i R^2- \gamma_i(1-s_i)\lambda n)\le 0,
\end{eqnarray*}
and hence $G^t\le 0$ for all $t$. By Lemma~\ref{lemma:dual-ascent-psdca}, this yields
\begin{eqnarray}\label{eqn:dual-ascent-psdca}
\E[D(\theta^t)-D(\theta^{t-1})]\ge \frac{s}{n}\E[P(\w^{t-1})-D(\theta^{t-1})].
\end{eqnarray}
Furthermore since
\begin{eqnarray*}
 P(\w^{t-1})-D(\theta^{t-1})\ge D(\theta^*)-D(\theta^{t-1}):=\epsilon^{t-1}_D,
\end{eqnarray*}
where $\theta^*$ the optimal solution of the dual problem and $D(\theta^t)-D(\theta^{t-1})=\epsilon^{t-1}_D - \epsilon^t_D$, we obtain that
\begin{eqnarray}\label{eqn:convergence-psdca}
\E[\epsilon^t_D] \le \left(1-\frac{s}{n}\right)\E[\epsilon^{t-1}_D]\le \left(1-\frac{s}{n}\right)^t\E[\epsilon^0_D].
\end{eqnarray}
In addition, since $P(0)=\frac{1}{n}\sum^n_{i=1}\phi_i(0)+\lambda r(0)\le 1$ and
\begin{eqnarray*}
D(0)=\frac{1}{n}\sum^n_{i=1}-\phi_i^*(0)-\lambda r^*(0) = \frac{1}{n}\sum^n_{i=1}-\max_{\u_i}(0-\phi_i(\u_i))-\max_{\u}(0-r(\u))=\frac{1}{n}\sum^n_{i=1}\min_{\u_i}\phi_i(\u_i)+\lambda r(\u)\ge 0,
\end{eqnarray*}
we have $\epsilon^0_D\le P(0)-D(0)\le 1$. Combining this with inequality~\eqref{eqn:convergence-psdca}, we obtain
\begin{eqnarray*}
\E[\epsilon^t_D]\le (1-\frac{s}{n})^t \le \exp(-\frac{st}{n})=\exp(-\frac{t}{n+\sum^n_{i=1}\frac{R^2}{\lambda n \gamma_i}}).
\end{eqnarray*}
where the equality $s=\frac{n}{n+\sum^n_{i=1}\frac{R^2}{\lambda n \gamma_i}}$ in equation~\eqref{eqn:distribution-ipsdca} is used.
According to the above inequality, by setting
\begin{eqnarray*}
t\ge\left(n+\sum^n_{i=1}\frac{R^2}{\lambda n \gamma_i}\right)\log(\frac{1}{\epsilon_D}) ,
\end{eqnarray*}
the proposed algorithm will achieve $\E[\epsilon^t_D]\le \epsilon_D$. Furthermore, according to inequality~\eqref{eqn:dual-ascent-psdca}
\begin{eqnarray*}
\E[P(\wt)-D(\theta^t)]\le \frac{n}{s}\E[\epsilon^t_D-\epsilon^{t+1}_D]\le \frac{n}{s}\E[\epsilon^t_D],
\end{eqnarray*}
by setting
\begin{eqnarray*}
t\ge(n+\sum^n_{i=1}\frac{R^2}{\lambda n \gamma_i})\log\left((n+\sum^n_{i=1}\frac{R^2}{\lambda n\gamma_i})\frac{1}{\epsilon_P}\right),
\end{eqnarray*}
we will obtain $\E[\epsilon^t_D]\le \frac{s}{n}\epsilon_P$ and $\E[P(\wt)-D(\theta^t)]\le \epsilon_P$.
\end{proof}
\textbf{Remark:} If we adopt uniform sampling, i.e., $p_i=1/n$ $\forall i$, then we have to use the same $\gamma$ for all $\phi_i$, which should be $\gamma_{min} = \min\{\gamma_1,\ldots,\gamma_n\}$ according to the analysis. Once replacing $\gamma_i$ with $\gamma_{min}$, this theorem will recover the conclusion in the theorem 1 of \cite{shalev2012proximal}, i.e., $T\ge (n+\frac{R^2}{\lambda \gamma_{min}})\log\left((n+\frac{R^2}{\lambda  \gamma_{min}})\frac{1}{\epsilon_P}\right)$ . However,
\begin{eqnarray*}
\frac{n+\frac{R^2}{\lambda \gamma_{min}}}{n+ \sum^n_{i=1}\frac{R^2}{\lambda \gamma_i n}}=
\frac{n\lambda \gamma_{min}+R^2}{n\lambda \gamma_{min}+ \frac{R^2}{n}\sum^n_{i=1}\frac{ \gamma_{min}}{ \gamma_i }}\ge 1,
\end{eqnarray*}
implies importance sampling does improve convergence, especially when $\sum^n_{i=1}\frac{ \gamma_{min}}{ \gamma_i } \ll  n $.

For non-smooth loss functions, the convergence rate for Proximal SDCA with importance sampling is given below.
\begin{thm}
Consider the proposed proximal SDCA with importance sampling. Assume that $\phi_i$ is $L_i$-Lipschitz and set $p_i=L_i/\sum^n_{j=1}L_j$, $\forall i\in\{1,\ldots,n\}$. To obtain an expected duality gap of $\E[P(\bar{\w})-D(\bar{\theta)}]\le\epsilon_P$ where $\bar{\w}=\frac{1}{T-T_0}\sum^T_{t=T_0+1}\w^{t-1}$ and $\bar{\theta}=\frac{1}{T-T_0}\sum^T_{t=T_0+1}\theta^{t-1}$, it suffices to have a total number of iterations of
\begin{eqnarray*}
T\ge \max(0, 2\lceil \frac{n}{\rho} \log(\frac{\lambda n}{\rho 2 R^2(\sum^n_{i=1}L_i)^2/n^2})\rceil) -n/\rho+\frac{20 R^2(\sum^n_{i=1}L_i)^2}{n^2\lambda\epsilon_P},
\end{eqnarray*}
where $\rho=\frac{n L_{min}}{\sum^n_{i=1}L_i}$. Moreover, when
\begin{eqnarray*}
t\ge \max(0, 2\lceil \frac{n}{\rho} \log(\frac{\lambda n}{\rho 2 R^2(\sum^n_{i=1}L_i)^2/n^2})\rceil)-2n/\rho+\frac{16 R^2(\sum^n_{i=1}L_i)^2}{n^2\lambda\epsilon_P} ,
\end{eqnarray*}
we have dual sub-optimality bound of $\E[D(\theta^*)-D(\theta^t)]\le\epsilon_P/2$.
\end{thm}
\begin{proof}
Since $p_i=L_i/\sum^n_{j=1}L_j$, the inequality~\eqref{eqn:Gt-bound} indicates $G^t\le G$ where $G=4 R^2(\sum^n_{i=1} L_i)^2/n^2$. Lemma~\ref{lemma:dual-ascent-psdca}, with $\gamma_i=0$, tells that
\begin{eqnarray}\label{eqn:duality-gap-lipschitz}
\E[D(\theta^t)-D(\theta^{t-1})]\ge \frac{s}{n}\E[P(\w^{t-1})-D(\theta^{t-1})]-(\frac{s}{n})^2\frac{G}{2\lambda},
\end{eqnarray}
for all $s\in[0, \rho]$, which further indicates
\begin{eqnarray*}
\E[\epsilon^t_D]\le(1-\frac{s}{n})\E[\epsilon^{t-1}_D]+(\frac{s}{n})^2\frac{G}{2\lambda},
\end{eqnarray*}
for all $s\in[0, \rho]$. Expanding the above inequality implies
\begin{eqnarray}\label{eqn:prox-sdca-lipschitz-i}
\E[\epsilon^t_D]\le (1-\frac{s}{n})^t\epsilon^0_D+(\frac{s}{n})^2\frac{G}{2\lambda}\sum^t_{\tau=1}(1-\frac{s}{n})^{\tau-1},
\end{eqnarray}
for all $s\in[0,\rho]$.

We next show that the above yields
\begin{eqnarray}\label{eqn:duality-gap-ipsdca}
\E[\epsilon^t_D]\le\frac{2 G}{\lambda(2n/\rho+t-t_0)},
\end{eqnarray}
for all $t\ge t_0=\max(0, \lceil \frac{n}{\rho} \log(\frac{2\lambda n}{\rho G})\rceil)$. Indeed, let us choose $s=\rho$, then at $t=t_0$, we have
\begin{eqnarray*}
\E[\epsilon^t_D]\le (1-\frac{\rho}{n})^t\epsilon^0_D+\frac{\rho^2 G}{n^2 2\lambda}\frac{1}{1-(1-\rho/n)}\le\exp(-\rho t/n)+\frac{\rho G}{2\lambda n}\le \frac{ G}{\lambda n/\rho}=\frac{2 G}{\lambda(2n/\rho+t_0-t_0)}
\end{eqnarray*}
where  the first inequality used the inequality~\eqref{eqn:prox-sdca-lipschitz-i}, the second inequality used the facts $(1-\frac{\rho}{n})^t\le \exp(-\rho t/n)$ and $\epsilon^0_D\le 1$, and the third inequality used the fact $t_0 \ge \lceil \frac{n}{\rho} \log(\frac{2\lambda n}{\rho G})\rceil$. This implies that the inequality~\eqref{eqn:duality-gap-ipsdca} holds at $t=t_0$. For $t> t_0$ we can use inductive argument. Suppose the claim holds for $t-1$, therefore
\[
\E[\epsilon^t_D]\le (1-\frac{s}{n})\E[\epsilon^{t-1}_D] + (\frac{s}{n})^2\frac{G}{2\lambda}\le(1-\frac{s}{n})\frac{2 G}{\lambda(2n/\rho+t-1-t_0)}+(\frac{s}{n})^2\frac{G}{2\lambda}.
\]
Choosing $s=\frac{2 n}{(2n/\rho+t-1-t_0)}\in[0,\rho]$ yields
\begin{eqnarray*}
\E[\epsilon^t_D]&\le&\left(1-\frac{2}{2n/\rho+t-1-t_0}\right)\frac{2 G}{\lambda(2n/\rho+t-1-t_0)} +\left(\frac{2}{2n/\rho+t-1-t_0}\right)^2\frac{G}{2\lambda}\\
&=&\frac{2 G}{\lambda(2n/\rho+t-1-t_0)}(1-\frac{1}{2n/\rho+t-1-t_0})\\
&=&\frac{2 G}{\lambda(2n/\rho+t-1-t_0)}(\frac{2n/\rho+t-1-t_0-1}{2n/\rho+t-1-t_0})\\
&\le& \frac{2 G}{\lambda(2n/\rho+t-1-t_0)}(\frac{2n/\rho+t-1-t_0}{2n/\rho+t-t_0})\\
&=&  \frac{2 G}{\lambda(2n/\rho+t-t_0)}.
\end{eqnarray*}
This provides a bound on the dual sub-optimality. We next turn to bound the duality gap. Summing  the inequality~\eqref{eqn:duality-gap-lipschitz} over $t=T_0+1,\ldots,T$ and rearranging terms we obtain that
\[
\E\left[\frac{1}{T-T_0}\sum^T_{t=T_0+1}(P(\w^{t-1})-D(\theta^{t-1}))\right]\le \frac{n}{s(T-T_0)}\E[D(\theta^T)-D(\theta^{T_0})]+\frac{s G}{2\lambda n}.
\]
Now if we choose $\bar{\w}$, $\bar{\theta}$ to be either the average vector or a randomly chosen vector over $t\in\{T_0+1,\ldots,T\}$, the the above implies
\[
\E[(P(\bar{\w})-D(\bar{\theta}))]\le \frac{n}{s(T-T_0)}\E[D(\theta^T)-D(\theta^{T_0})]+\frac{s G}{2\lambda n}.
\]
If $T\ge n/\rho+T_0$ and $T_0\ge t_0$, we can set $s=n/(T-T_0)\le \rho$ and combining with~\eqref{eqn:duality-gap-ipsdca} we obtain
\begin{eqnarray*}
\E[(P(\bar{\w})-D(\bar{\theta}))]&\le& \frac{n}{s(T-T_0)}\E[D(\theta^T)-D(\theta^{T_0})]+\frac{s G}{2\lambda n}\\
&\le&\E[D(\theta^*)-D(\theta^{T_0})]+\frac{ G}{2\lambda(T-T_0)}\\
&\le& \frac{2G}{\lambda(2n/\rho+T_0-t_0)}+\frac{ G}{2\lambda(T-T_0)}.
\end{eqnarray*}
A sufficient condition for the above to be smaller than $\epsilon_P$ is that $T_0\ge\frac{4G}{\lambda\epsilon_P}-2n/\rho+t_0$ and $T\ge T_0+\frac{G}{\lambda\epsilon_P}$. It also implies that $\E[D(\theta^*)-D(\theta{T_0})]\le\epsilon_P/2$. Since we also need $T_0\ge t_0$ and $T-T_0\ge n/\rho$, the overall number of required iterations can be
\begin{eqnarray*}
T_0\ge\max\{t_0, 4G/(\lambda\epsilon_P)-2n/\rho+t_0\},\quad T-T_0\ge \max\{n/\rho, G/(\lambda\epsilon_P)\}.
\end{eqnarray*}
Using the fact $a+b\ge \max(a, b)$ concludes the proof of this theorem.
\end{proof}
{\bf Remark:} When we replace all the $L_i$ with $L_{max}=\max\{L_1,\ldots,L_n\}$, the above theorem will still be valid, and the sampling distribution becomes the uniform distribution. In this case we will recover Theorem 2 of~\cite{shalev2012proximal}, i.e., $T\ge  \max(0, 2\lceil n \log(\frac{\lambda n}{2R^2 L_{max}^2})\rceil) -n+\frac{20 R^2(L_{max})^2}{\lambda\epsilon_P}$. However, the ratio between the leading terms is
\begin{eqnarray*}
\frac{(L_{max})^2}{(\sum^n_{i=1}L_i)^2/n^2}= (\frac{n}{\sum^n_{i=1}L_i/L_{max}})^2\ge 1 ,
\end{eqnarray*}
which again implies that the importance sampling strategy will improve convergence, especially when $(\sum^n_{i=1}\frac{L_i}{L_{max}})^2\ll n^2$.

%============================================================================================
%
%  Applications
%
%============================================================================================

\section{Applications}
\label{sec:application}
There are numerous possible applications of our proposed algorithms. Here we will list several popular applications. In this section, we will set $\psi(\w)=\frac{1}{2}\|\w\|_2^2$, so that the stochastic mirror descent is stochastic gradient descent.

\subsection{Hinge Loss Based SVM with $\ell_2$ Regularization}
Suppose our task is to solve the typical Support Vector Machine (SVM):
\begin{eqnarray*}
\min_{\w} \frac{1}{n}\sum^n_{i=1} [1-y_i\w^\top\x_i]_+ +\frac{\lambda}{2}\|\w\|_2^2.
\end{eqnarray*}
Assume that $X=\max_i\|\x_i\|_2$ is not too large, such as for text categorization problems where each $\x_i$ is a bag-of-words representation of some short document. To solve this problem we can use two different proximal SGD or proximal SDCA  with importance sampling  as follows.

\subsubsection{Proximal SGD with Importance Sampling}
We can set regularizer as $r(\w)=0$, and the loss function as $\phi_i(\w)=[1-y_i\w^\top\x_i]_+ +\frac{\lambda}{2}\|\w\|_2^2$, which is $\lambda$-strongly convex, so that
\[
\prox_{\lambda r}(\x)=\arg\min_{\w}\left(0+\frac{1}{2}\|\w-\x\|_2^2\right)=\x,
\]
and
\[
\nabla \phi_i(\w)= - \sign([1-y_i\w^\top\x_i]_+) y_i\x_i+\lambda\w .
\]
Using $P(\w^*)= D(\theta^*)$, we find  that the optimal solution of SVM satisfies $\|\w^*\|_2\le 1/\sqrt{\lambda}$. So we can project the iterative solutions into $\{\w\in\R^d|\|\w\|_2\le 1/\sqrt{\lambda}\}$ using Euclidean distance, while the theoretical analysis is still valid. In this way, we have $ \|\nabla\phi_i(\w)\|_2\le \|\x_i\|_2+\sqrt{\lambda}$. According to our analysis, we should set
\[
p_i=\frac{\|\x_i\|_2+\sqrt{\lambda}}{\sum^n_{j}(\|\x_j\|_2+\sqrt{\lambda})}.
\]

\subsubsection{Proximal SDCA with Importance Sampling}
We can set $r(\w)=\frac{1}{2}\|\w\|_2^2$, and loss function as $\phi_i(\w)=[1-y_i\w^\top\x_i]_+$  which is $\|\x_i\|_2$-Lipschitz. According to our analysis, the distribution should be
\begin{eqnarray*}
p_i=\frac{\|\x_i\|_2}{\sum^n_{i=1}\|\x_i\|_2}.
\end{eqnarray*}

Furthermore, it is easy to get the dual function of $\phi_i$ as
\begin{displaymath}
\phi^*_i(-\theta)= \left\{ \begin{array}{ll}
-\alpha  & \textrm{$\theta=\alpha y_i \x_i$, $\alpha\in [0,1]$}\\
 \infty& \textrm{otherwise}
\end{array} \right.
\end{displaymath}
Given the dual function, using $\theta_i=\alpha_i y_i \x_i$, the options I in the algorithm produces a closed form solution as
\begin{eqnarray*}
\Delta\theta_i=\max\left(-\alpha_i,\min\left(1-\alpha_i,\frac{1-y_i\x_i^\top\w^{t-1}}{\|\x_i\|_2^2/(\lambda n)}\right)\right)y_i\x_i.
\end{eqnarray*}

\subsection{Squared Hinge Loss Based SVM with $\ell_2$ Regularization}
Suppose our interest is to solve the task of optimizing squared hinge loss based Support Vector Machine (SVM) with $\ell_2$ regularization:
\begin{eqnarray*}
\min_{\w} \frac{1}{n}\sum^n_{i=1}\left([1-y_i\w^\top\x_i]_+\right)^2 +\frac{\lambda}{2}\|\w\|_2^2.
\end{eqnarray*}

\subsubsection{Proximal SGD with Importance Sampling}
Firstly, using the inequality $P(\w^*)=D(\theta^*)$, we can get $\|\w^*\|_2\le 1/\sqrt{\lambda}$. So we can project the iterative solutions into $\{\w\in\R^d|\|\w\|_2\le 1/\sqrt{\lambda}\}$ using Euclidean distance, while the previous theoretical analysis is still valid.

 If we set $r(\w)=0$ and $\phi_i(\w)= \left([1-y_i\w^\top\x_i]_+\right)^2+\frac{\lambda}{2}\|\w\|_2^2$ so that $\prox_{\lambda r}(\x)=\x$ and
\begin{eqnarray*}
\nabla \phi_i(\w)=- 2[1-y_i\w^\top\x_i]_+ y_i\x_i+\lambda\w.
\end{eqnarray*}
Because $\|\nabla \phi_i(\w)\|_2\le 2(1+\|\x_i\|_2/\sqrt{\lambda})\|\x_i\|_2+\sqrt{\lambda}$, according the previous analysis, the optimal distribution for this case should be
\begin{eqnarray*}
p_i=\frac{2(1+\|\x_i\|_2/\sqrt{\lambda})\|\x_i\|_2+\sqrt{\lambda}}{\sum^n_{j=1}[2(1+\|\x_j\|_2/\sqrt{\lambda})\|\x_j\|_2+\sqrt{\lambda}]}.
\end{eqnarray*}

\subsubsection{Proximal SDCA with Importance Sampling}
If we set $r(\w)=\frac{1}{2}\|\w\|^2$, which is $1$-strongly convex with $\|\cdot\|_{P'}=\|\cdot\|_2$, we have $\phi_i(\w)=\left([1-y_i\w^\top\x_i]_+\right)^2$, which is $(2\|\x_i\|_2^2)$-smooth with respect to $\|\cdot\|_{P}=\|\cdot\|_2$. As a result, the optimal distribution for proximal SDCA with importance sampling should be
\begin{eqnarray*}
p_i=(1+\frac{2\|\x_i\|_2^2}{\lambda n})/(n+\sum^n_{j=1}\frac{2\|\x_j\|_2^2}{\lambda n}),
\end{eqnarray*}
where we used the fact $R=\sup_{\u\not=0}\|\u\|_{D'}/\|\u\|_D=1$.

It can be derived that the dual function of $\phi(\cdot)$ is
\begin{displaymath}
\phi^*_i(-\theta)= \left\{ \begin{array}{ll}
-\alpha + \alpha^2/4 & \textrm{$\theta=\alpha y_i \x_i$, $\alpha\ge 0$}\\
 \infty& \textrm{otherwise}
\end{array} \right.
\end{displaymath}

Plugging the above equation into the update, we can observe that option I in the algorithm has the following closed-form solution:
\begin{eqnarray*}
\Delta \theta_i =\max\left(\frac{1-y_i\w^\top\x_i-\alpha_i/2}{1/2+\|\x_i\|_2^2/(\lambda n)},\ -\alpha_i \right)y_i\x_i.
\end{eqnarray*}

\subsection{Squared Hinge Loss Based SVM with $\ell_1$ Regularization}
Suppose our interest is to solve the task of optimizing squared hinge loss based Support Vector Machine (SVM) with $\ell_1$ regularization:
\begin{eqnarray*}
\min_{\w} \frac{1}{n}\sum^n_{i=1}\left([1-y_i\w^\top\x_i]_+\right)^2 + \lambda \|\w\|_1.
\end{eqnarray*}
where $\ell_1$ regularization is introduced to make the optimal model sparse, which can alleviate the effect of the curse of dimensionality, and improve the model interpretability.

\subsubsection{Proximal SGD with Importance Sampling}
Using the inequality $P(\w^*)\le P(0)$, we obtain the fact that $\|\w^*\|_2\le \|\w^*\|_1\le 1/\lambda$. So we can project the iterative solutions into $\{\w\in\R^d|\|\w\|_2\le 1/\lambda\}$ using Euclidean distance, while the previous analysis is still valid.

If we set regularizer as $ r(\w)=\|\w\|_1$, then loss function is $\phi_i(\w)= \left([1-y_i\w^\top\x_i]_+\right)^2$, so that the proximal mapping is
\[
\prox_{\lambda r}(\x)=\sign(\x)\odot[|\x|-\lambda]_+,
\]
and
\[
\nabla \phi_i(\w)=- 2 \left([1-y_i\w^\top\x_i]_+\right)y_i\x_i.
\]
Because $\nabla\phi_i(\w)\le 2(1+\|\x_i\|_2/\lambda)\|\x_i\|_2$, according to the previous analysis, we should set
\[
p_i=\frac{(1+\|\x_i\|_2/\lambda)\|\x_i\|_2}{\sum^n_{j=1}(1+\|\x_j\|_2/\lambda)\|\x_j\|_2}.
\]

\subsubsection{Proximal SDCA with Importance Sampling}
Let $\w^*$ be the optimal solution, which satisfies $\|\w^*\|_2\le 1/\lambda$. Choosing $\delta=\lambda^2\epsilon$ and
\[
r(\w)=\frac{1}{2}\|\w\|_2^2+\frac{\lambda}{\delta}\|\w\|_1,
\]
which is $1$-strongly convex with respect to $\|\cdot\|_{P'}=\|\cdot\|_2$.

Consider the problem
\[
\min_\w \widehat{P}(\w) =\left[\frac{1}{n}\sum^n_{i=1}\phi_i(\w)+\delta r(\w)\right],
\]
then if $\w$ is an $\epsilon/2$ approximated solution of the above problem , it holds that
\[
\frac{1}{n}\sum^n_{i=1}\phi_i(\w)+\lambda\|\w\|_1 \le \widehat{P}(\w)\le \widehat{P}(\w^*)+\frac{\epsilon}{2}\le \frac{1}{n}\sum^n_{i=1}\phi_i(\w^*)+\lambda\|\w^*\|_1+\epsilon,
\]
which implies $\w$ is an $\epsilon$ approximated solution of the original optimization problem.

When we adopt Proximal SDCA with importance sampling to minimize $\hat{P}(\w)$, we have $\phi_i(\w)=\left([1-y_i\w^\top\x_i]_+\right)^2$,
which is $(2\|\x_i\|^2)$-smooth with respect to $\|\cdot\|_P=\|\cdot\|_2$. As a result, the optimal distribution for Proximal SDCA with importance sampling should be
\begin{eqnarray*}
p_i=(1+\frac{2\|\x_i\|_2^2}{\lambda n})/(n+\sum^n_{j=1}\frac{2\|\x_j\|_2^2}{\lambda n}).
\end{eqnarray*}

Plugging $\phi^*_i(-\theta)= -\alpha + \alpha^2/4$, $\theta=\alpha y_i \x_i$, $\alpha\ge 0$ into the algorithm, we can observe that the option I produces the following closed-form solution:
\begin{eqnarray*}
\Delta \theta_i =\max\left(\frac{1-y_i\w^\top\x_i-\alpha_i/2}{1/2+\|\x_i\|_2^2/(\lambda n)},\ -\alpha_i \right)y_i\x_i.
\end{eqnarray*}

Finally it is easy to verify that
\begin{eqnarray*}
\nabla r^*(\v)=\sign(\v)\odot[|\v|-\frac{\lambda}{\delta}]_+.
\end{eqnarray*}

\section{Experimental Results}
\label{sec:experiment}
In this section, we evaluate the empirical performance of the proposed algorithms.

\subsection{Experimental Testbed and Setup}
To compare our algorithms with their traditional versions without importance sampling, we focus on the task of optimizing squared hinge loss based SVM with $\ell_2$ regularization:
\begin{eqnarray*}
\min_{\w} \frac{1}{n}\sum^n_{i=1}\left( [1-y_i\w^\top\x_i]_+\right)^2 +\frac{\lambda}{2}\|\w\|_2^2.
\end{eqnarray*}

We compared our Iprox-SGD with traditional prox-SGD (actually Pegasos~\cite{DBLP:conf/icml/Shalev-ShwartzSS07}), and Iprox-SDCA with prox-SDCA (actually SDCA~\cite{DBLP:journals/jmlr/ShaiTong13}). For Iprox-SGD, we adopt the method in the subsection 5.2.1, while for Iprox-SDCA, we adopt the method in the subsection 5.2.2.

\begin{table}[h]
\begin{center}
\vspace{-0.2in}
\caption{Datasets used in the experiments.}\label{tab:datasets}
\begin{tabular}{|c|c|c|c|}        \hline
{ Dataset}  &{Dataset Size}   & {Features}
\\\hline\hline
ijcnn1           & 49990        & 22        \\
kdd2010(algebra) & 8407752      & 20216830  \\
w8a              & 49749        & 300
\\\hline
\end{tabular}
\vspace{-0.1in}
\end{center}
\end{table}
To  evaluate the performance of our algorithms, the experiments were performed on several real world datasets, which are chosen fairly randomly in order to cover various aspects of datasets. All the datasets can be downloaded from LIBSVM website\footnote{\url{http://www.csie.ntu.edu.tw/~cjlin/libsvmtools/}}. The details of the dataset characteristics are provided in the Table~\ref{tab:datasets}.

To make a fair comparison, all algorithms adopt the same setup in our experiments. In particular, the regularization parameter $\lambda$ of SVM is set as  $10^{-4}$, $10^{-6}$, $10^{-4}$ for  ijcnn1, kdd2010(algebra), and w8a, respectively. For prox-SGD and Iprox-SGD, the step size is set as $\eta_t = 1/(\lambda t)$ for all the datasets.

Given these parameters, we estimated the ratios between the constants in the convergence bounds for uniform sampling and the proposed importance sampling strategies, which is listed in the table~\ref{tab:ratio}.
\begin{table}[h]
\begin{center}
\vspace{-0.2in}
\caption{Theoretical Constant Ratios for The Experiment Datasets.}\label{tab:ratio}
\begin{tabular}{|c|c|c|c|c|}        \hline
Constant Ratio & ijcnn1   &  kdd2010 & w8a
\\\hline\hline
 $\frac{n\sum^n_{i=1}(G_i)^2}{(\sum^n_{i=1}G_i)^2}$ (for SGD) & 1.0643        &   1.4667      &  1.9236\\
$\frac{n\lambda \gamma_{min}+R^2}{n\lambda \gamma_{min}+ \frac{R^2}{n}\sum^n_{i=1}\frac{ \gamma_{min}}{ \gamma_i }}$ (for SDCA)  &  1.1262      &1.1404   & 1.3467  \\
\hline
\end{tabular}
\vspace{-0.1in}
\end{center}
\end{table}
These ratios imply that the importance sampling will be effective for SGD on kdd2010 and w8a, but not very effective for ijcnn1, which will be verified by later empirical results. In addition, these ratios imply the importance sampling will accelerate the minimization of the duality gap for all the datasets, which will also be demonstrated by later experiments.

All the experiments were conducted by fixing 5 different random seeds for each dataset. All the results were reported by averaging over these 5 runs. We evaluated the learning performance by measuring primal objective value ($P(\w^t)$) for SGD algorithms, and duality gap value ($P(\w^t)-D(\theta^t)$) for SDCA algorithms. In addition, to examine the generalization ability of the learning algorithm, we also evaluated the test error rate. Finally, we also reported the variances of the stochastic gradients of the two algorithms to check the effectiveness of importance sampling. Finally, for Iprox-SGD and Iprox-SDCA, the uniform sampling is adopted at the first epoch, so that the performance is the same with SGD and SDCA there, respectively.

\subsection{Evaluation on Iprox-SGD}
The figure~\ref{fig:ISGD} summarized  experimental results in terms of primal objective values, test error rates and variances of the stochastic gradients varying over the learning process  on all the datasets for SGD and Iprox-SGD.

\begin{figure}[htp]
\begin{center}
\includegraphics[width=2.1in]{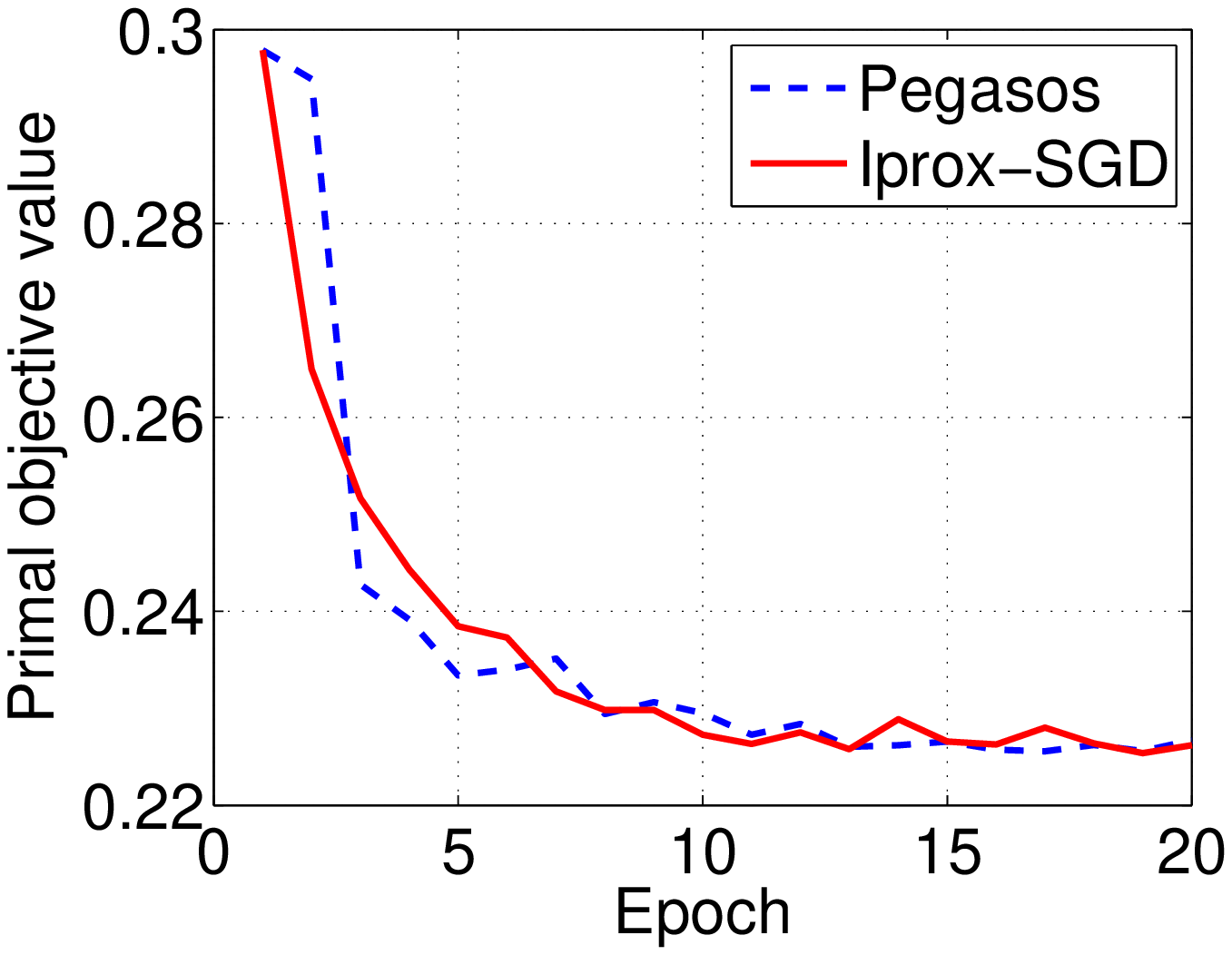}
\includegraphics[width=2.1in]{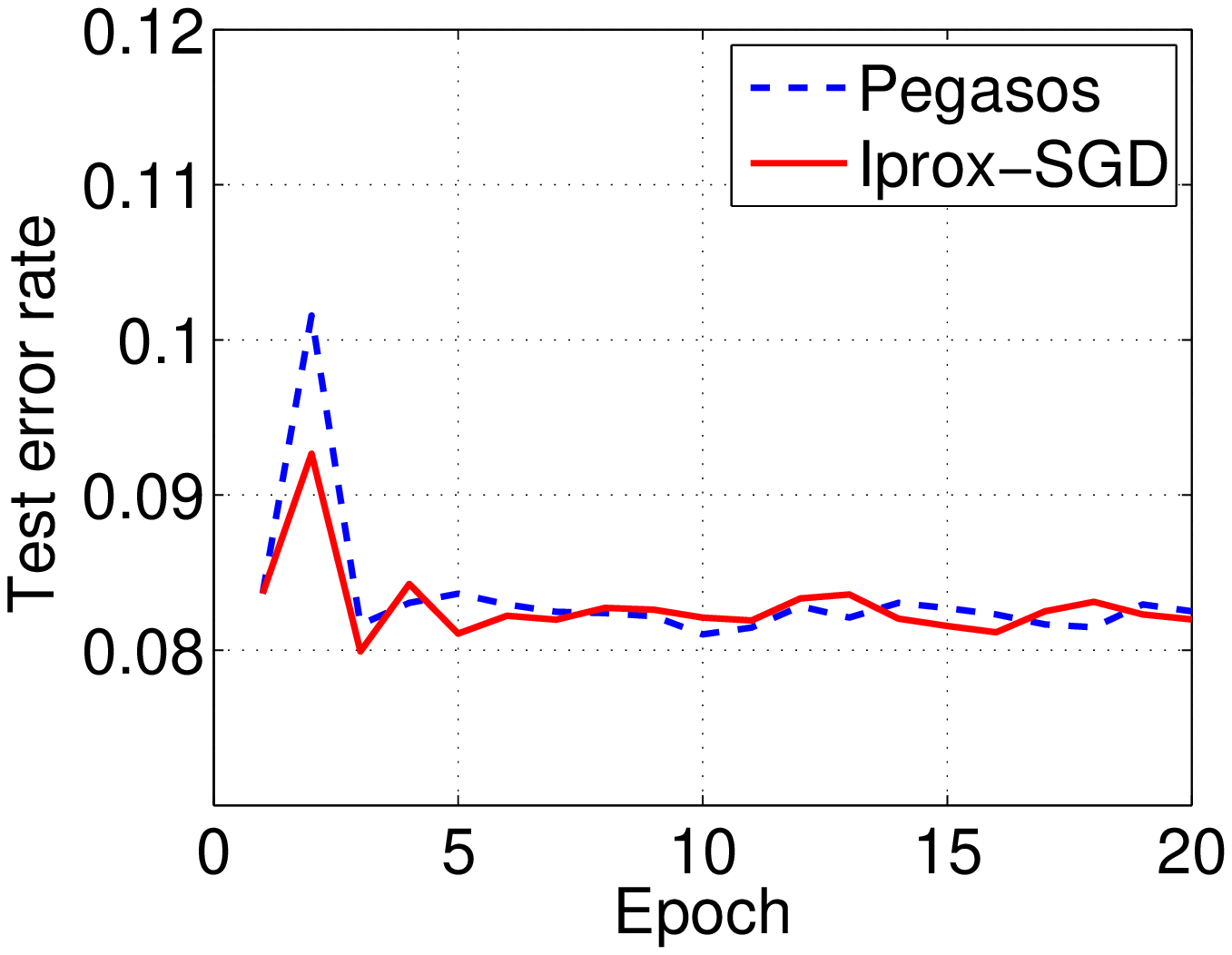}
\includegraphics[width=2.1in]{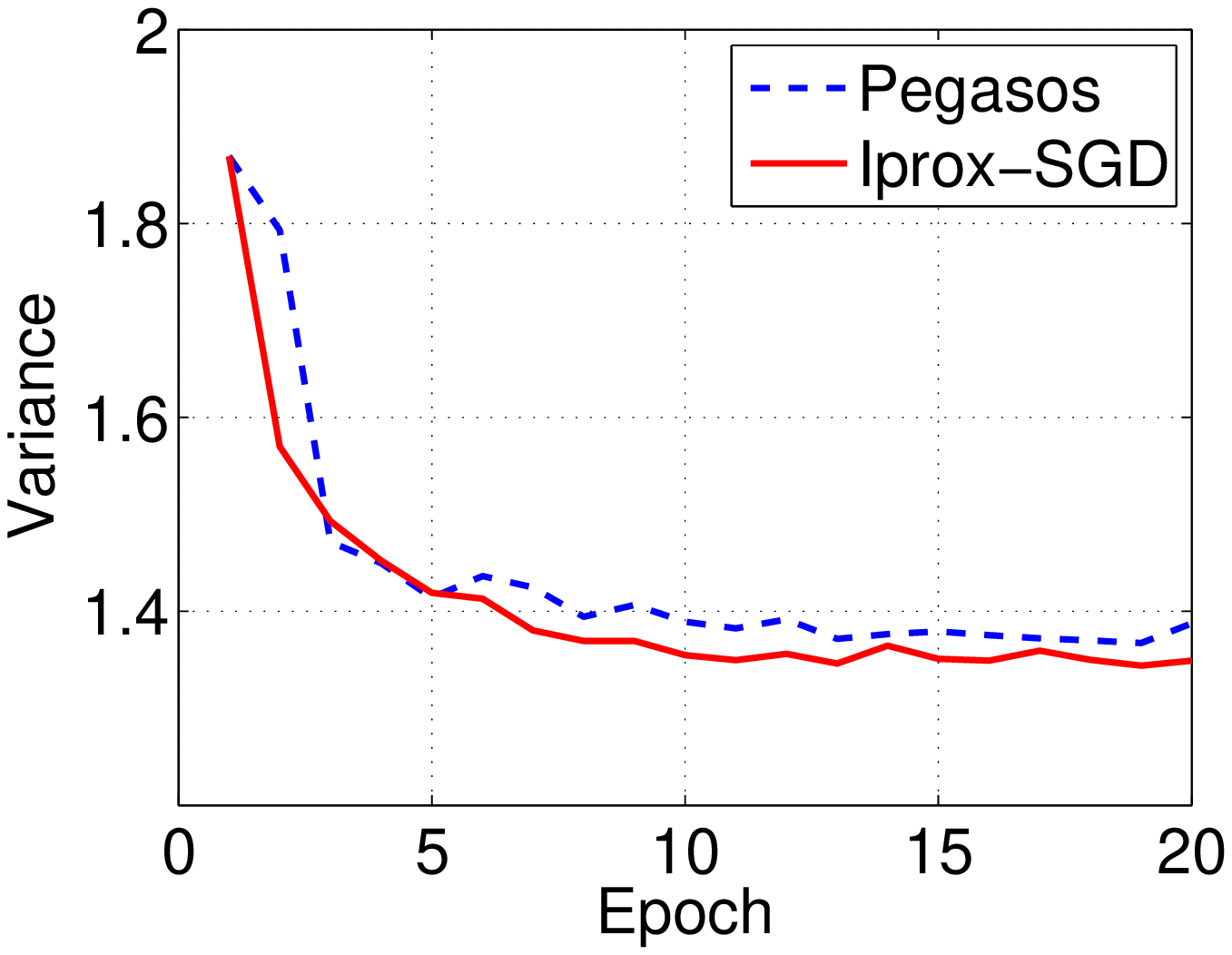}
{\scriptsize \makebox[2.1in]{(a)~primal objective value on {\bf ijcnn1}}~\makebox[2.1in]{(b)~test error rate on {\bf ijcnn1}}~\makebox[2.1in]{(c)~variance on {\bf ijcnn1}}}
\end{center}
\begin{center}
\includegraphics[width=2.1in]{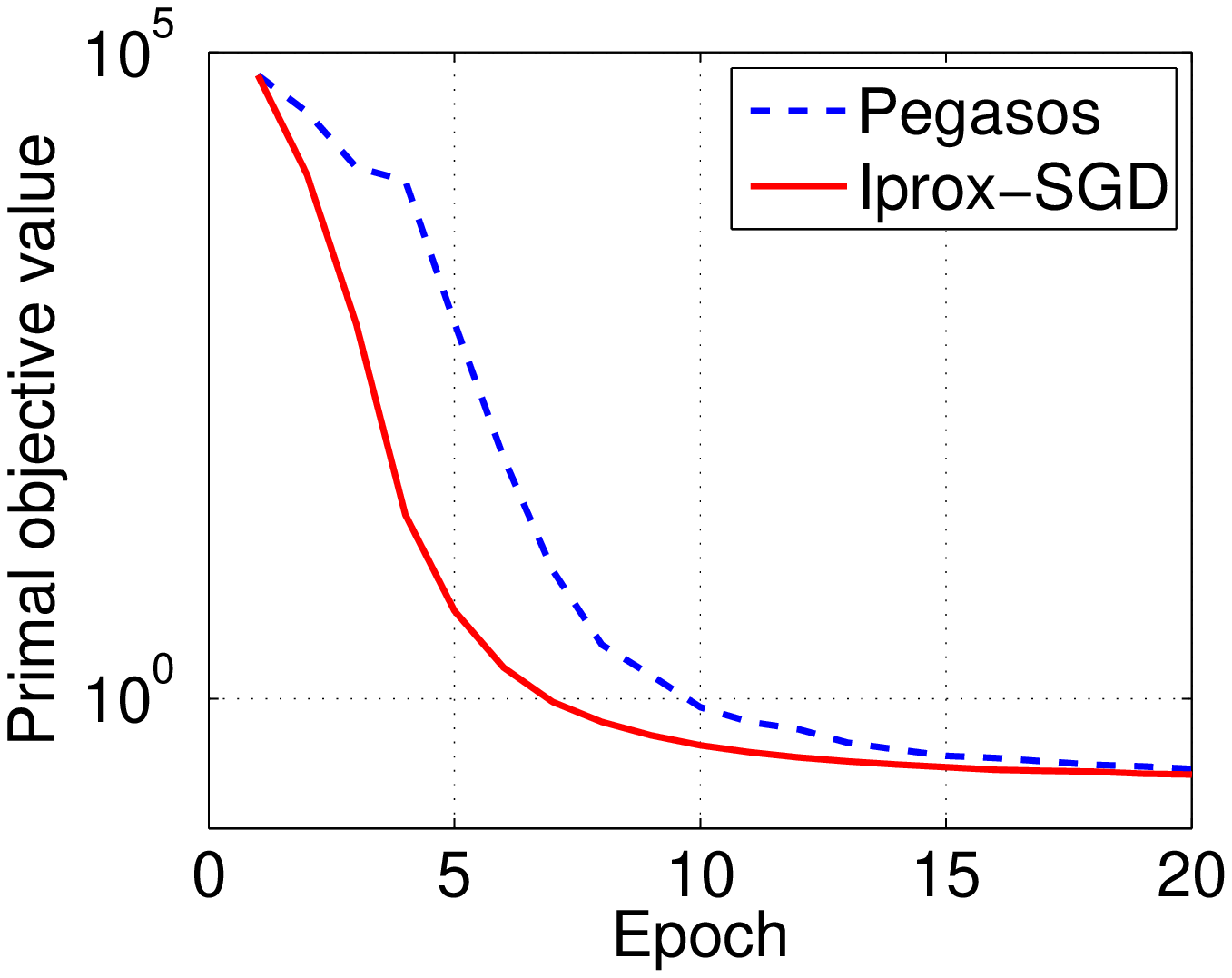}
\includegraphics[width=2.1in]{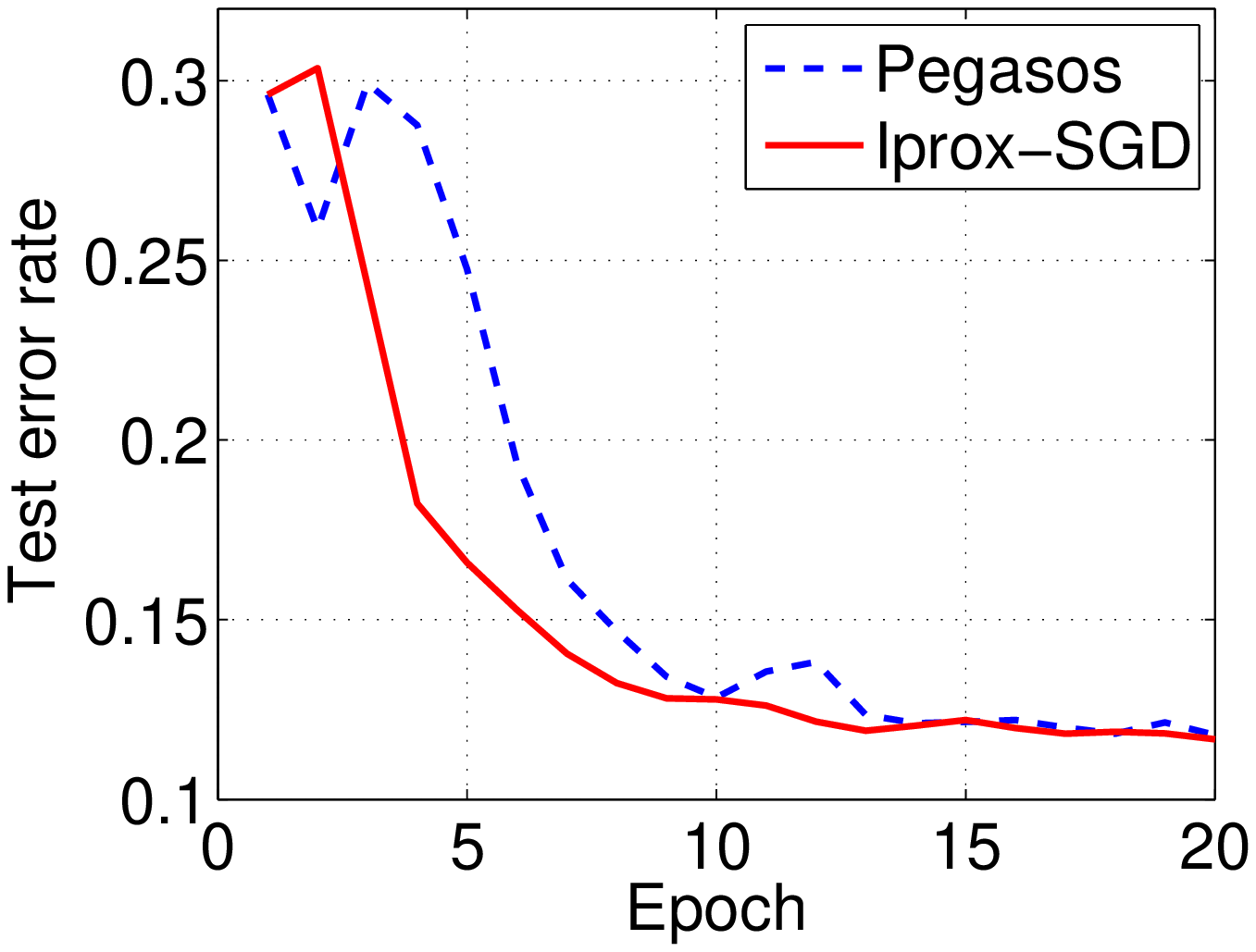}
\includegraphics[width=2.1in]{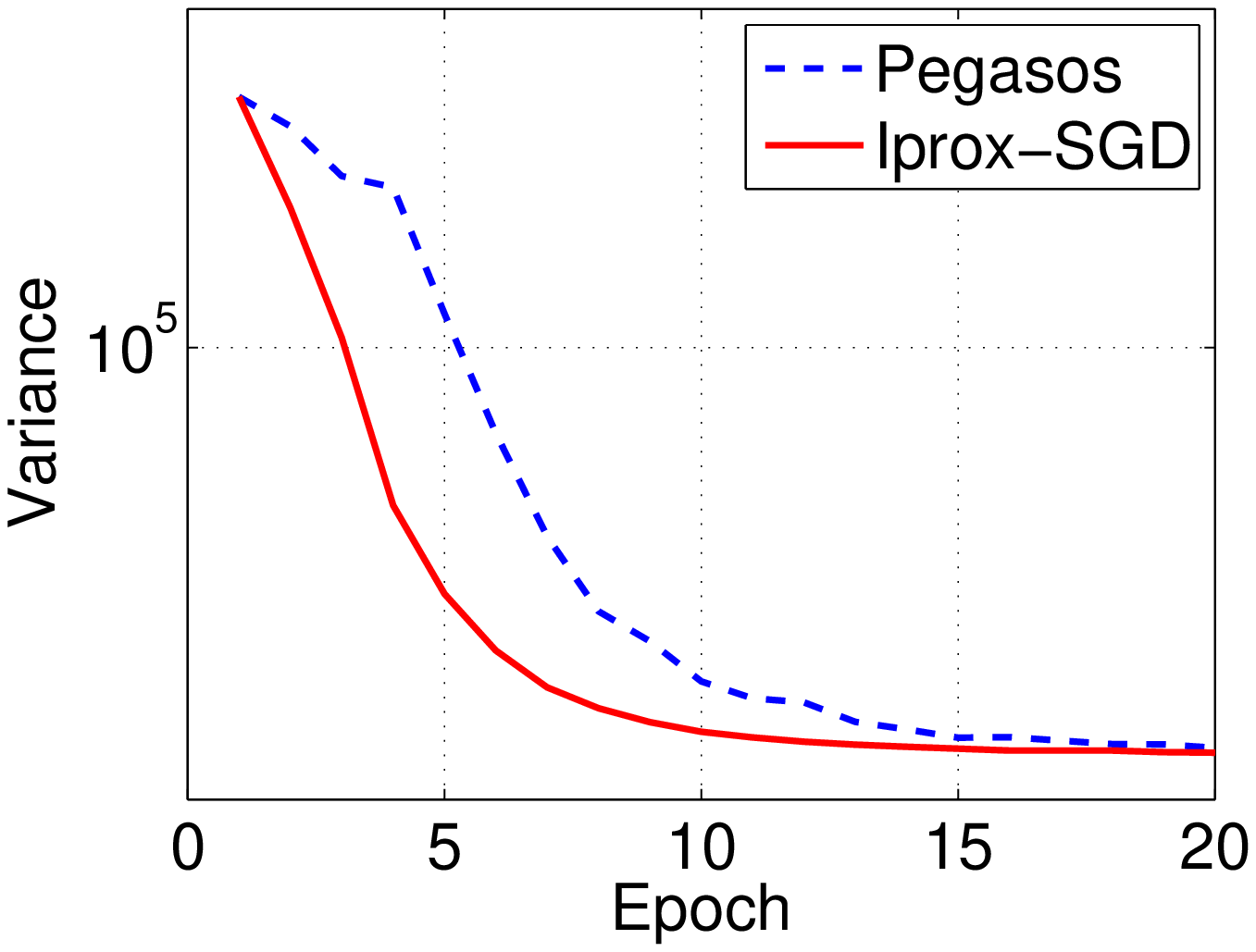}
{\scriptsize \makebox[2.1in]{(d)~primal objective value on {\bf kdd2010}}~\makebox[2.1in]{(e)~test error rate on {\bf kdd2010}}~\makebox[2.1in]{(f)~variance on {\bf kdd2010}}}
\end{center}
\begin{center}
\includegraphics[width=2.1in]{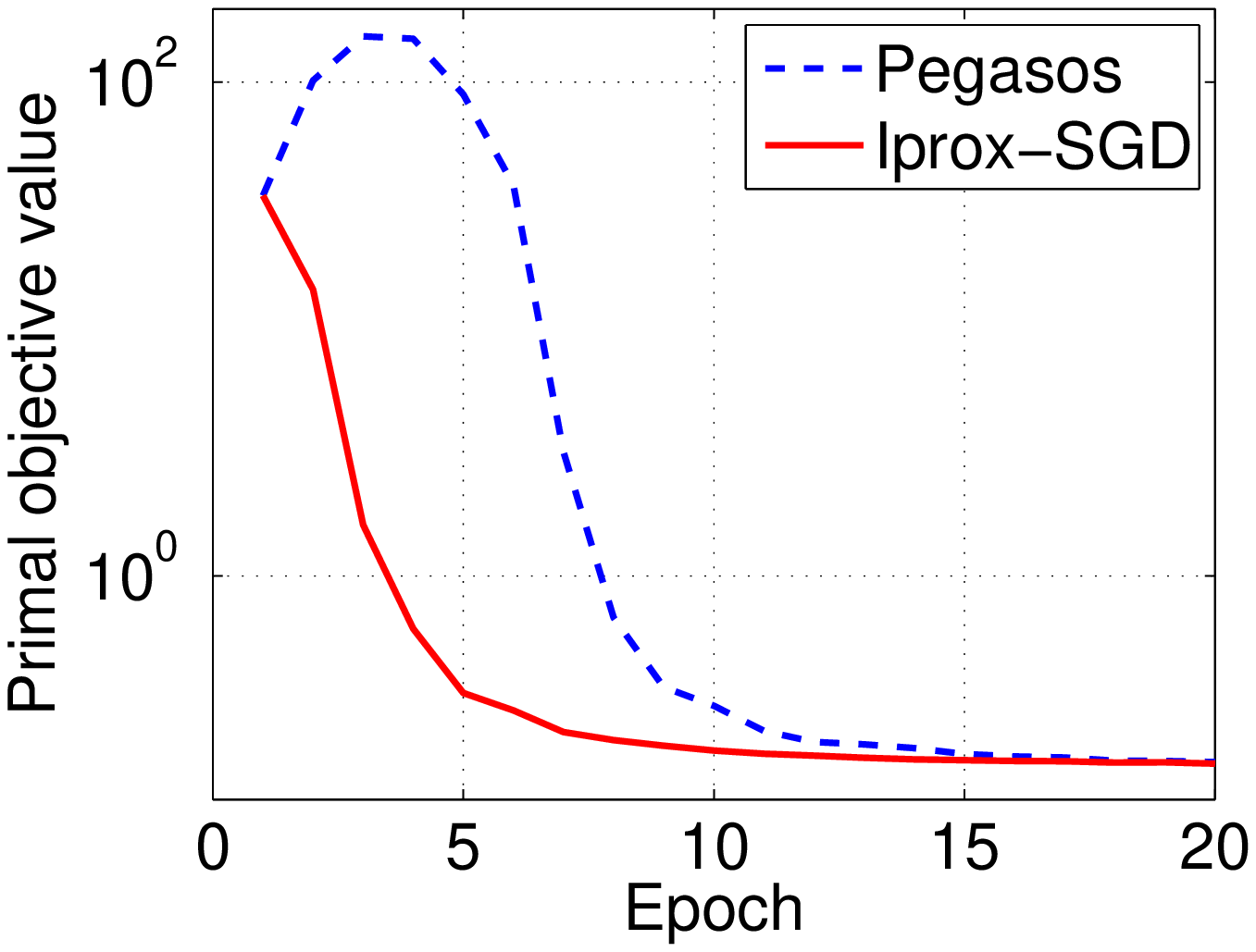}
\includegraphics[width=2.1in]{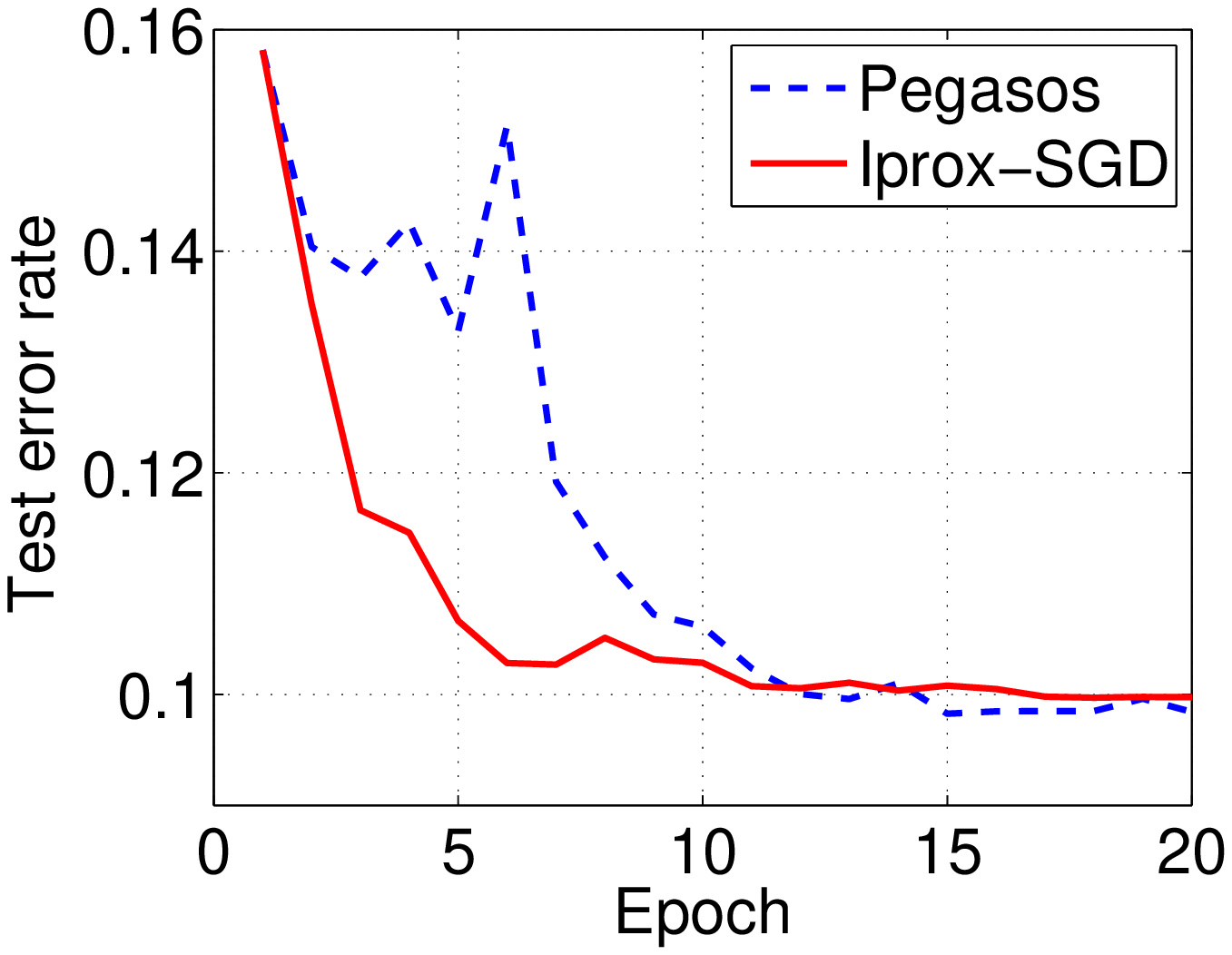}
\includegraphics[width=2.1in]{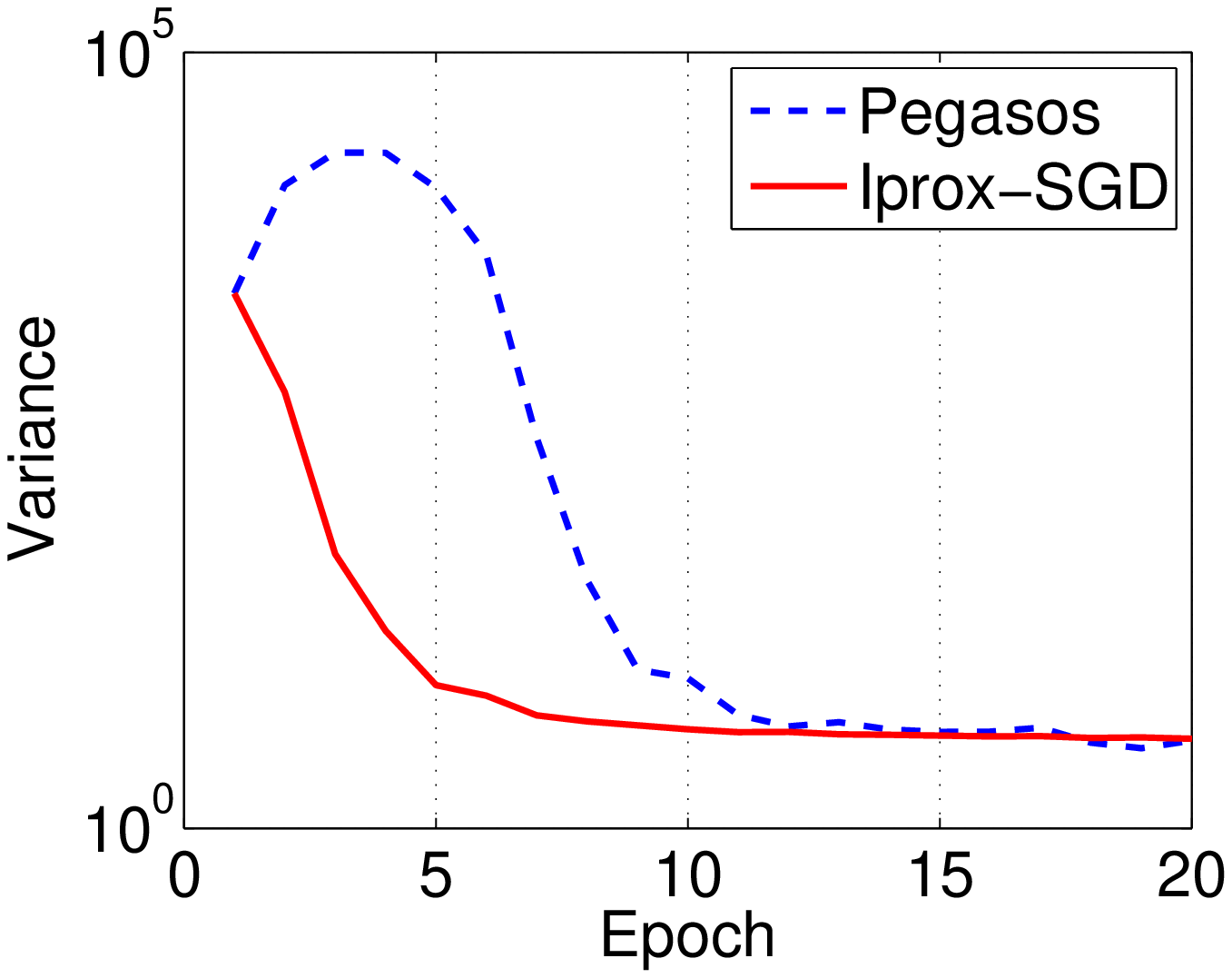}
{\scriptsize \makebox[2.1in]{(g)~primal objective value on {\bf w8a}}~\makebox[2.1in]{(h)~test error rate on {\bf w8a}}~\makebox[2.1in]{(i)~variance on {\bf w8a}}}
\end{center}
\caption{Comparison between Pegasos with Iprox-SGD on several datasets. Epoch for the horizontal axis is the number of iterations divided by dataset size.}
\label{fig:ISGD}
\end{figure}

First, the left column summarized the primal objective values of Iprox-SGD in comparison to SGD with uniform sampling on all the datasets. On the last two datasets,, the proposed Iprox-SGD algorithm achieved the fastest convergence rates. Because these two algorithms adopted the same learning rates, this observation implies that the proposed importance sampling does sampled more informative stochastic gradient during the learning process. Second, the central column summarized the test error rates of the two algorithms, where Iprox-SGD achieves significantly smaller test error rates than those of SGD on the last two dataset. This indicates that the proposed importance sampling approach is effective in improving generalization ability. In addition, the right column shows the variances of stochastic gradients for the Iprox-SGD and SGD algorithms, where we can observe Iprox-SGD enjoys much smaller variances than SGD on the last two dataset. This again demonstrates that the proposed importance sampling strategy is effective in reducing the variance of the stochastic gradients. Finally, on the first dataset, the proposed Iprox-SGD algorithm achieved comparable convergence rate compared with traditional prox-SGD, which indicates that Iprox-SGD may degenerate into the traditional prox-SGD when the variance of training dataset is significantly small.

\subsection{Evaluation on Iprox-SDCA}
The figure~\ref{fig:ISDCA} summarized  experimental results in terms of duality gap values, test error rates and variances of the stochastic gradients varying over the learning process  on all the datasets for SDCA and Iprox-SDCA.

\begin{figure}[htbp]
\begin{center}
\includegraphics[width=2.1in]{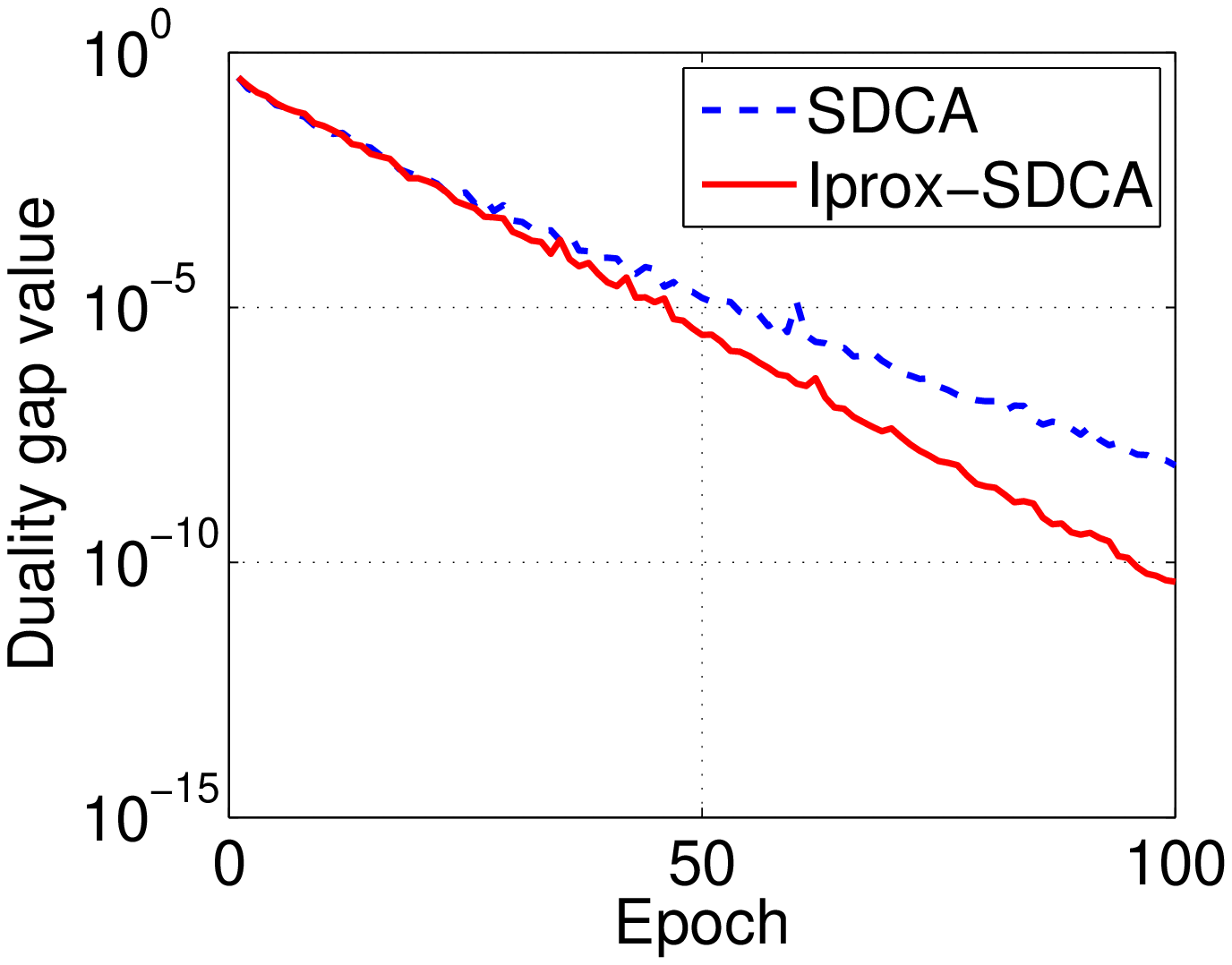}
\includegraphics[width=2.1in]{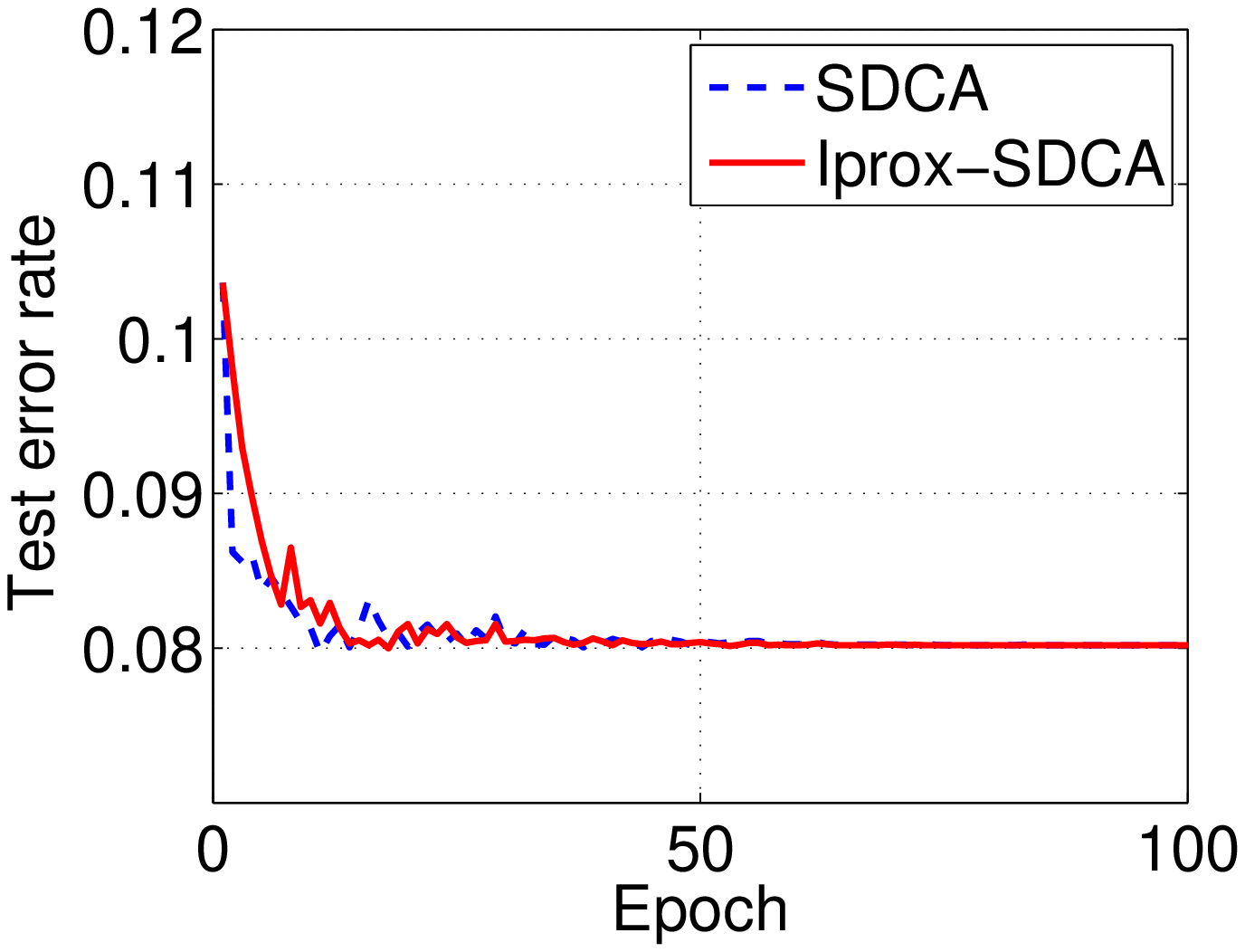}
\includegraphics[width=2.1in]{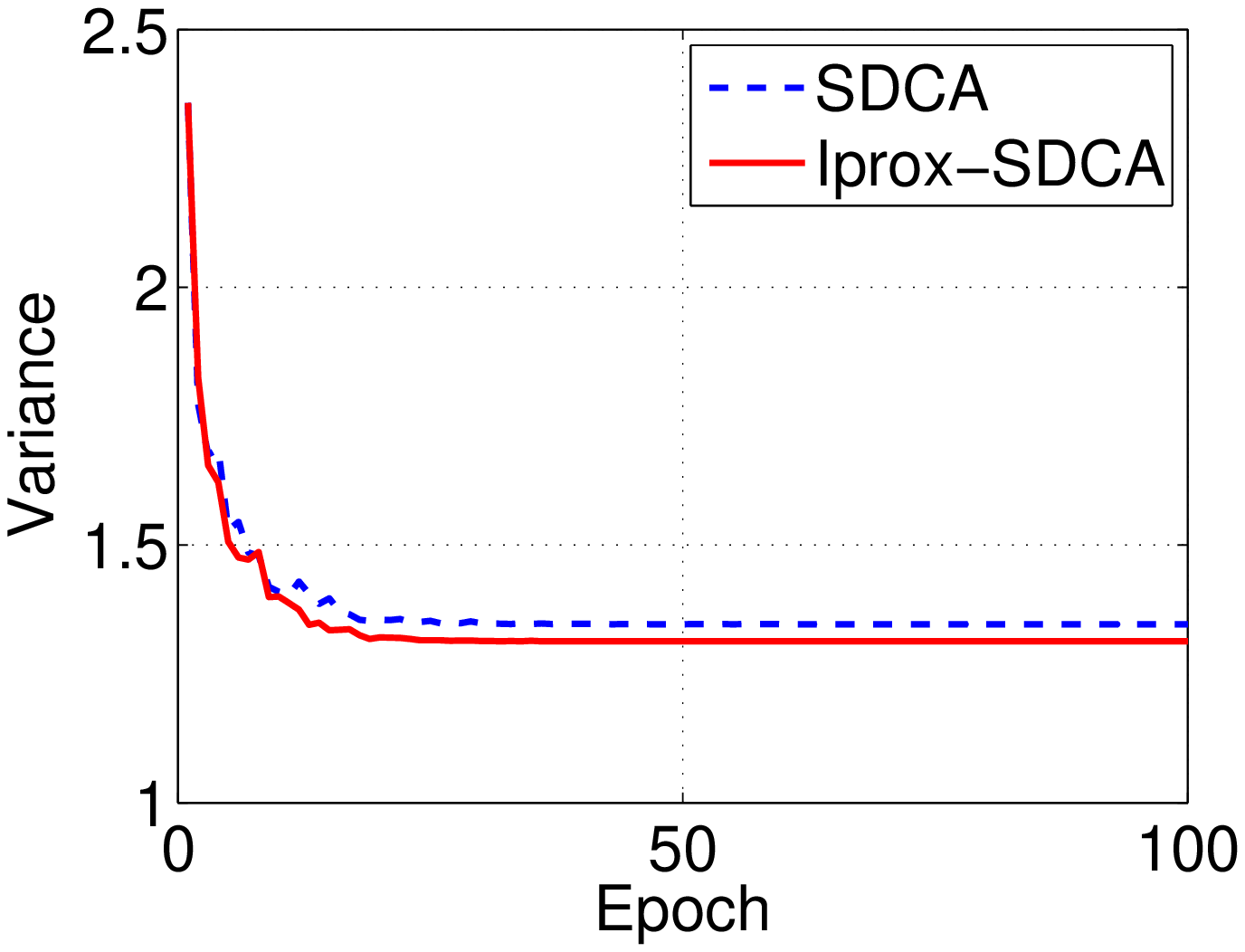}
{\scriptsize \makebox[2.1in]{(a)~duality gap value on {\bf ijcnn1}}~\makebox[2.1in]{(b)~test error rate on {\bf ijcnn1}}~\makebox[2.1in]{(c)~variance on {\bf ijcnn1}}}
\end{center}
\begin{center}
\includegraphics[width=2.1in]{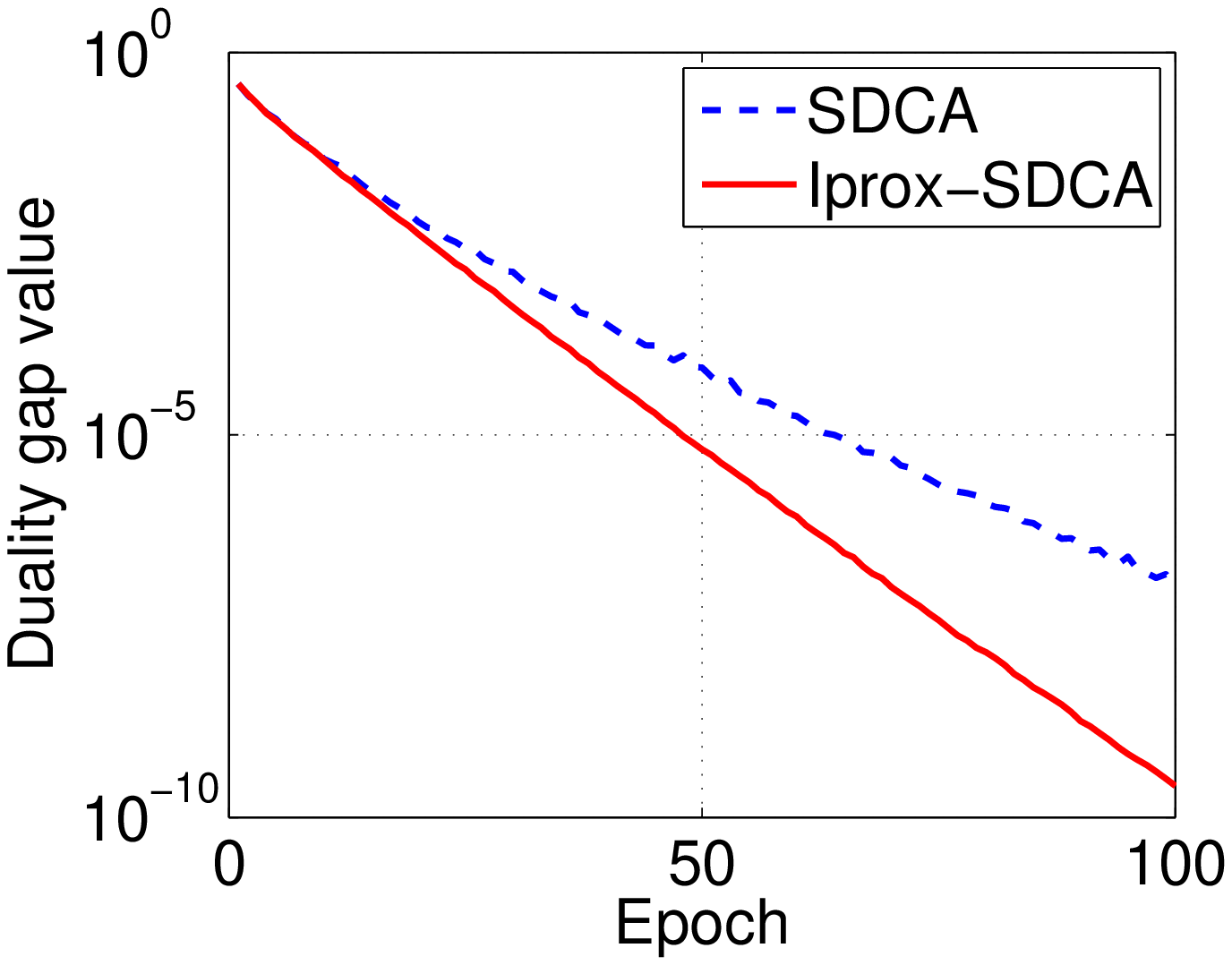}
\includegraphics[width=2.1in]{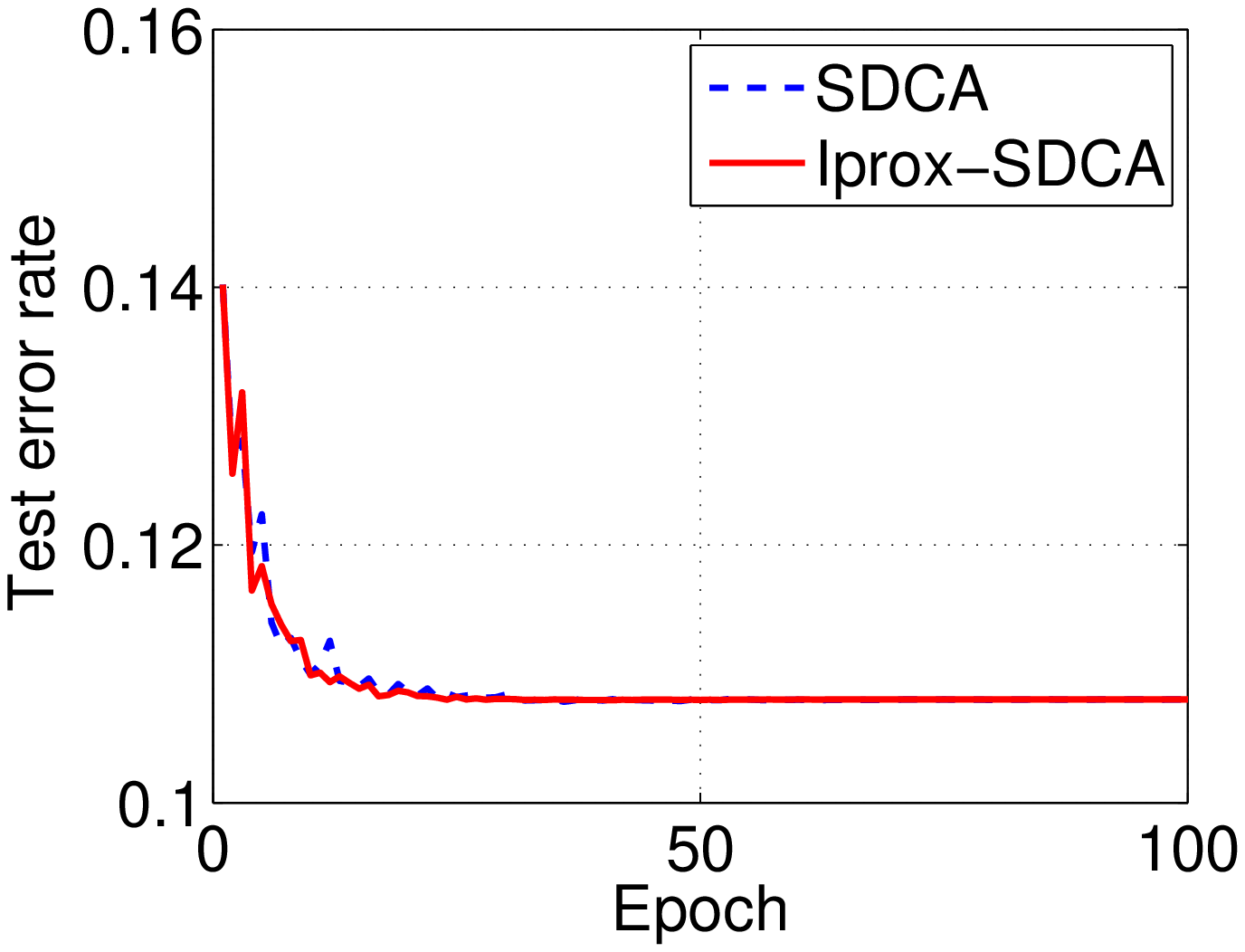}
\includegraphics[width=2.1in]{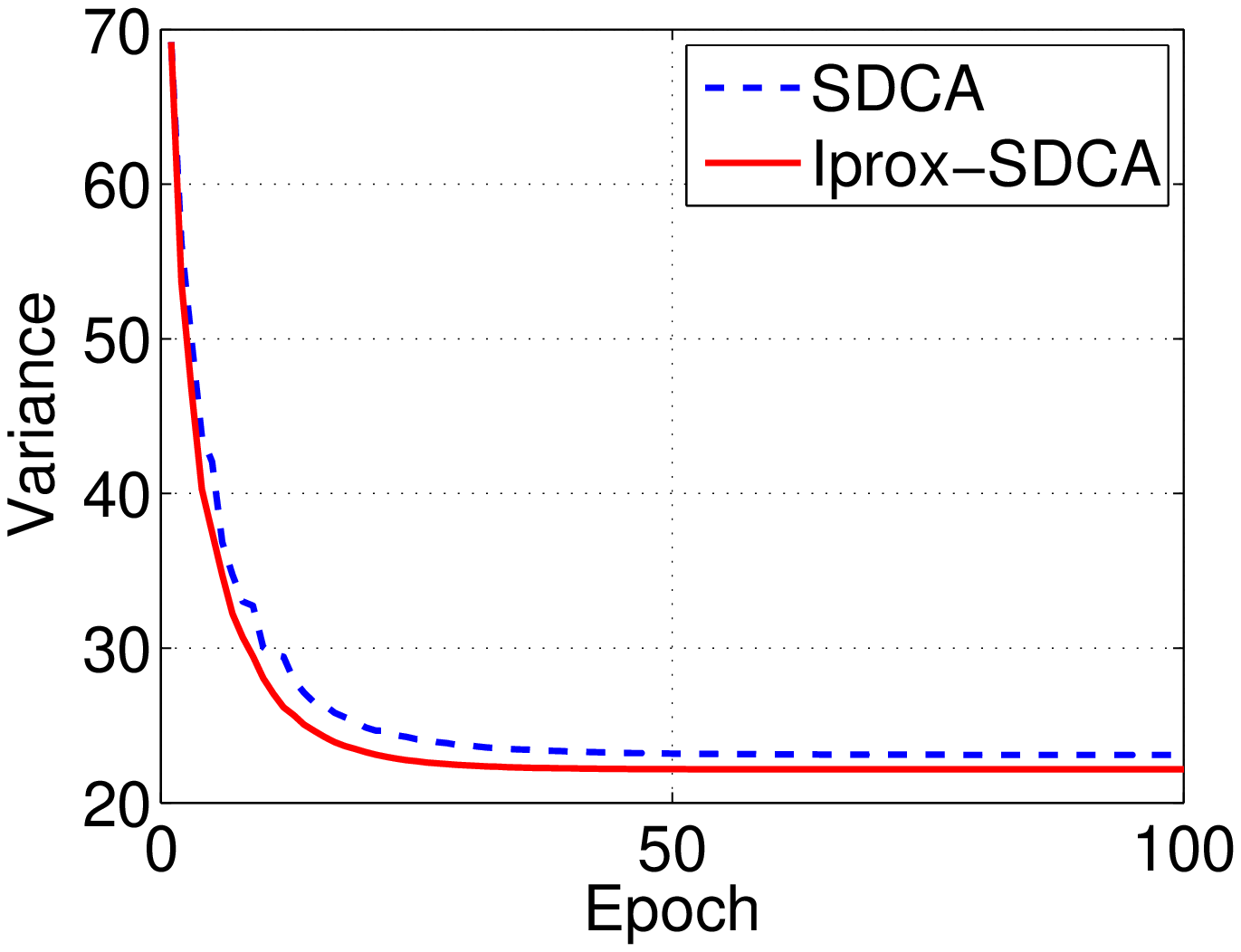}
{\scriptsize \makebox[2.1in]{(d)~duality gap value on {\bf kdd2010}}~\makebox[2.1in]{(e)~test error rate on {\bf kdd2010}}~\makebox[2.1in]{(f)~variance on {\bf kdd2010}}}
\end{center}
\begin{center}
\includegraphics[width=2.1in]{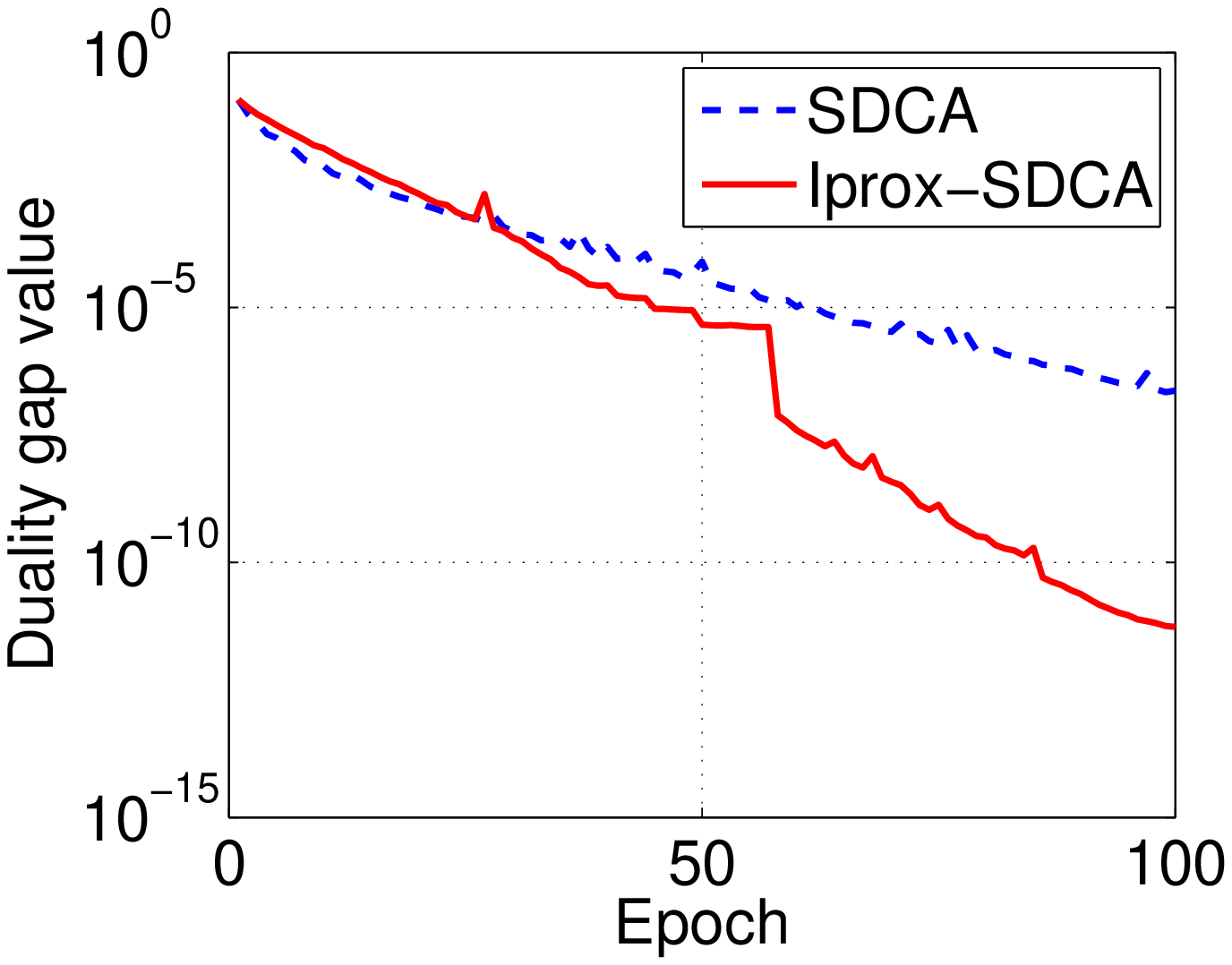}
\includegraphics[width=2.1in]{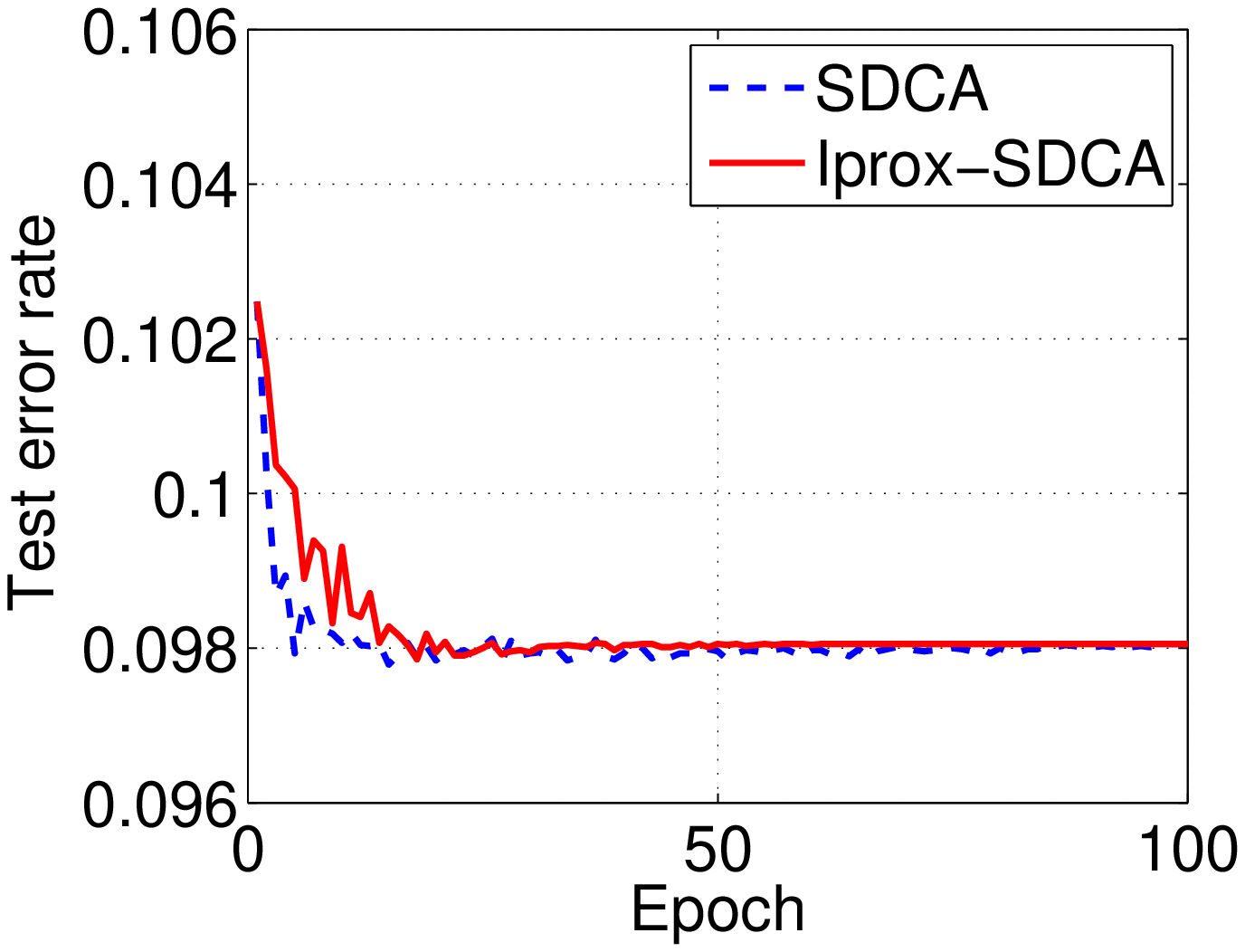}
\includegraphics[width=2.1in]{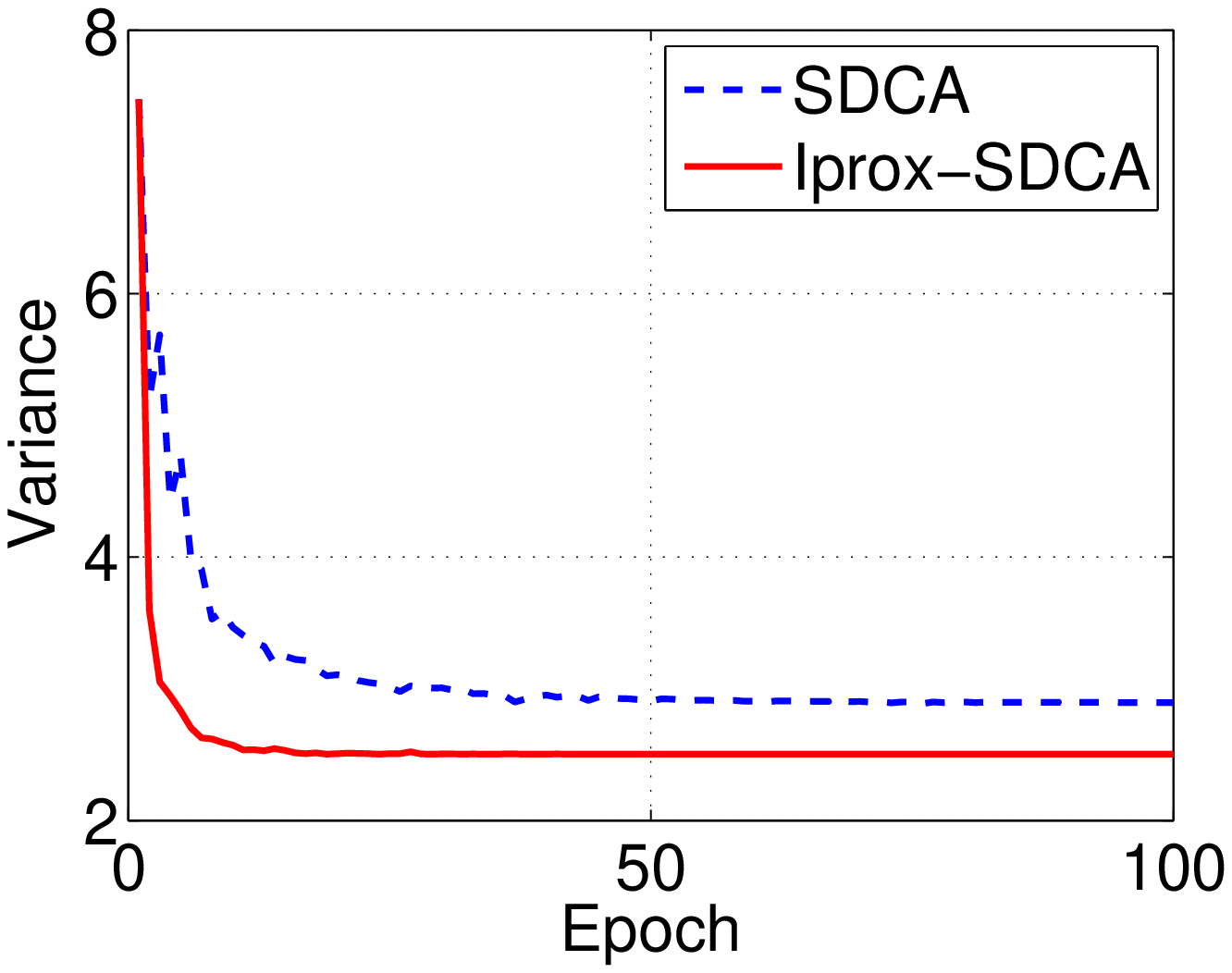}
{\scriptsize \makebox[2.1in]{(g)~duality gap value on {\bf w8a}}~\makebox[2.1in]{(h)~test error rate on {\bf w8a}}~\makebox[2.1in]{(i)~variance on {\bf w8a}}}
\end{center}
\caption{Comparison between SDCA with Iprox-SDCA on several datasets. Epoch for the horizontal axis is the number of iterations divided by dataset size.}
\label{fig:ISDCA}
\end{figure}

We have several observations from these empirical results. First, the left column summarized the dual gap values of Iprox-SDCA in comparison to SDCA with uniform sampling on all the datasets. According to the dual gap values on all the datasets, the proposed Iprox-SDCA algorithm converged much faster than the standard SDCA, which indicates that the proposed importance sampling strategy make the duality gap minimization more efficient during the learning process. Second, the central column summarized the test error rates of the two algorithms, where the test error rates of Iprox-SDCA is comparable with those of SGD on all the dataset. The results indicate that SDCA is quite fast at the first few epochs so that the importance sampling does not improve the test accuracy, although importance sampling can accelerate the minimization of duality gap. In addition, the right column shows the variances of stochastic gradients for the Iprox-SDCA and SDCA algorithms, where we can observe Iprox-SDCA enjoys a bit smaller variances than SDCA on all the dataset. However the improvement is too small so that the test error rate is not significantly reduced. This might due to that SDCA is a kind of stochastic variance reduction gradient method~\cite{DBLP:conf/nips/Johnson013}.

\section{Conclusion}
\label{sec:conclusion}
This paper studied stochastic optimization with importance sampling that reduces the variance. Specifically we considered importance sampling strategies for Proximal Stochastic Gradient Descent and for Proximal Stochastic Dual Coordinate Ascent. For prox-SGD with importance sampling, our analysis showed that in order to reduce variance, the sample distribution should depend on the norms of the gradients of the loss functions, which can be relaxed to the smooth constants or the Lipschitz constants of all the loss functions;  for prox-SDCA with importance sampling, our analysis showed that the sampling distribution should rely on the smooth constants or Lipschitz constants of all the loss functions. Compared to the traditional prox-SGD and prox-SDCA methods with uniform sampling, we showed that the proposed importance sampling methods can significantly improve the convergence rate of optimization under suitable situations. Finally, a set of experiments confirm our theoretical analysis.

\section*{Appendix}
\begin{lemma}\label{lem:nonexpansive}
Let $r$ be a convex function, and $\psi$ be a $\sigma$-strongly convex function w.r.t. $\|\cdot\|$. Assume
\begin{eqnarray*}
\uh = \arg\min_{\w}\left[\langle \g_u , \w\rangle + \lambda r(\w) + \frac{1}{\eta}\Bpsi(\w, \z) \right],\quad  \vh = \arg\min_{\w}\left[\langle \g_v, \w\rangle + \lambda r(\w) + \frac{1}{\eta}\Bpsi(\w, \z) \right],
\end{eqnarray*}
then, we have
\begin{eqnarray*}
\|\uh-\vh\|\le \frac{\eta}{\sigma}\|\g_u -\g_v\|_* .
\end{eqnarray*}
\end{lemma}
\begin{proof}
Firstly, using the definition of Bregman divergence and the optimality of $\uh$ we have
\begin{eqnarray*}
\mathbf{a}=\nabla\psi(\z)-\nabla\psi(\uh)-\eta \g_u \in \eta \lambda\partial r(\uh).
\end{eqnarray*}
Similarly, we also have
\begin{eqnarray*}
\mathbf{b}=\nabla\psi(\z)-\nabla\psi(\vh)-\eta  \g_v \in \eta \lambda\partial r(\vh).
\end{eqnarray*}
Since $\eta \lambda r(\w)$ is convex, its subdifferential is a monotone operator, so we have
\begin{eqnarray*}
0\le \langle\mathbf{a}-\mathbf{b}, \uh-\vh\rangle =  \langle\nabla\psi(\vh)-\nabla\psi(\uh)+\eta \g_v-\eta \g_u, \uh-\vh\rangle.
\end{eqnarray*}
Re-arranging the above inequality and using the $\sigma$-strongly convexity of $\psi$, we can get
\begin{eqnarray*}
\langle\eta \g_v-\eta \g_u, \uh-\vh\rangle \ge \langle\nabla\psi(\uh) - \nabla\psi(\vh), \uh-\vh\rangle\ge \sigma \|\uh-\vh\|^2.
\end{eqnarray*}
Using Cauchy-Schwartz inequality for the left-hand side results in
\begin{eqnarray*}
\eta\|\g_v-\g_u\|_*\|\uh-\vh\| \ge \sigma \|\uh-\vh\|^2.
\end{eqnarray*}
Dividing both side by $\sigma\|\uh-\vh\|$ concludes the proof.
\end{proof}

{
\bibliography{reference}
\bibliographystyle{unsrt}
}

\end{document}